%% file: main.tex
\PassOptionsToPackage{final}{graphicx}
\PassOptionsToPackage{final}{listings}
\documentclass{llncs}
\usepackage[utf8]{inputenc}
\usepackage[vlined,linesnumbered]{algorithm2e}
\usepackage{amsopn}
\usepackage{amsmath}
\usepackage{amsfonts}
\usepackage{thmtools,thm-restate}
\usepackage{bm}
\usepackage{caption} 
\usepackage{sansmath}
\usepackage{amssymb}
\usepackage{amsmath}
\usepackage{amssymb}
\usepackage{amsfonts}
\usepackage{amsopn}
\usepackage{amscd}
\usepackage[bookmarks=false]{hyperref}
\usepackage{verbatim}
\usepackage{subfig}
\usepackage{tikz}
\usepackage{proof}
\usepackage{alltt}
\usepackage{color}
\usepackage{xspace}
\usepackage{wrapfig}
\usepackage{rotating}
\usepackage{xcolor,colortbl}
\usepackage[outdir=./]{epstopdf}
\usepackage{sansmath}
\usepackage[square,numbers,sectionbib]{natbib}
\usepackage[inline]{enumitem}

\input{macros}
\title{Approximating Euclidean by Imprecise Markov Decision Processes}

\author{Manfred Jaeger
\and Giorgio Bacci
\and Giovanni Bacci
\and Kim Guldstrand Larsen
\and Peter Gj{\o}l Jensen
}

\institute{Department of Computer Science, Aalborg University, Denmark }

\begin{document}
\maketitle

\begin{abstract}
  Euclidean Markov decision processes are a powerful tool for modeling control
  problems under uncertainty over continuous domains.  Finite state imprecise,
  Markov decision processes can be used to approximate the behavior of these
  infinite models. In this paper we address two questions: first, we investigate
  what kind of approximation guarantees are obtained when the Euclidean process is approximated
  by finite state approximations induced by 
  increasingly fine partitions of the continuous state space. We show that for
  cost functions over finite time horizons the approximations become arbitrarily
  precise. Second, we use  imprecise Markov decision process approximations as
  a tool to analyse and validate cost functions and strategies obtained by reinforcement
  learning. We find that, on the one hand, our new theoretical results validate basic design choices
  of a previously proposed reinforcement learning approach. On the other hand, the imprecise
  Markov decision process approximations reveal some inaccuracies in the learned cost functions.
\end{abstract}

\section{Introduction}

Markov Decision  Processes (MDP) \cite{Puterman2005} provide  a unifying
framework for  modeling decision  making in situations  where outcomes
are  partly  random  and  partly  under  the  control  of  a  decision
maker. MDPs are  useful for studying optimization  problems solved via
dynamic  programming and  reinforcement  learning.  They are used  in
several areas,  including economics, control, robotics  and autonomous
systems.   In its simplest form, an MDP  comprises a finite
set of states $\states$, a finite set of control actions $\act$, which for each
state $s$ and action $a$  specifies the transition probabilities $P_a(s,s')$
to successor states  $s'$.  In  addition, transitioning  from a
state $s$  an action  $a$ has  an immediate  cost $C(s,a)$\footnote{In
  several  alternative but  essentially equivalent  definitions of  MDPs
  transitions have associated rewards rather than cost, and the reward
  may be  depend on  the successor state as  well.}. The overall
problem  is to  find a  strategy $\sigma$  that specifies  the action
$\sigma(s)$  to be  made  in  state $s$  in  order  to optimize  some
objective (e.g. the expected cost of reaching a goal state).

For  many applications,  however,  such as  queuing systems,  epidemic
processes (e.g. COVID19), and  population processes the restriction to
a finite state-space is inadequate.   Rather, the underlying system has
an infinite state-space and the decision making process must take into
account  the continuous  dynamics of  the system.   In this  paper, we
consider a  particular class of infinite-state  MDPs, namely Euclidean
Markov Decision  Processes \cite{jaeger2019teaching}, where the  state space $\states$
is  given  by a  (measurable)  subset  of  $\reals^K$ for  some  fixed
dimension  $K$.

As   an  example,   consider  the   semi-random  walk   illustrated  on
the left of
Fig.~\ref{randomwalkfig}                with               state-space
$\states=[0,x_{max}]\times [0,t_{max}]$  (one dimensional space,  and time).
Here the goal is to cross  the $x=1$ finishing line before $t=1$.  The
decision  maker has  two actions  at her  disposal: to  move fast  and
expensive (cost $3$), or to  move slow and cheap (cost $1$).  
Both  actions have uncertainty
about distance traveled  and time taken.  This  uncertainty is modeled
by a uniform distribution over a successor state square: given current
state $(x,t)$ and action $a \in \{\mathit{slow},\mathit{fast}\}$, the distribution over possible successor
states       is       the        uniform       distribution       over
$[x+\delta(a) -\varepsilon,x+\delta(a) + \varepsilon ] \times [t+\tau(a) - \varepsilon ,t+\tau(a) + \varepsilon]$,
where $(\delta(a),\tau(a))$ represents the direction of the movement in space and time which depends
on the action $a$, while the parameter $\varepsilon$ models the uncertainty. 
Now, the question is  to find the strategy $\sigma:\states\rightarrow \act$ that will minimize the expected
cost of reaching a goal state.

\begin{figure}
\centering
\includegraphics[scale=0.3]{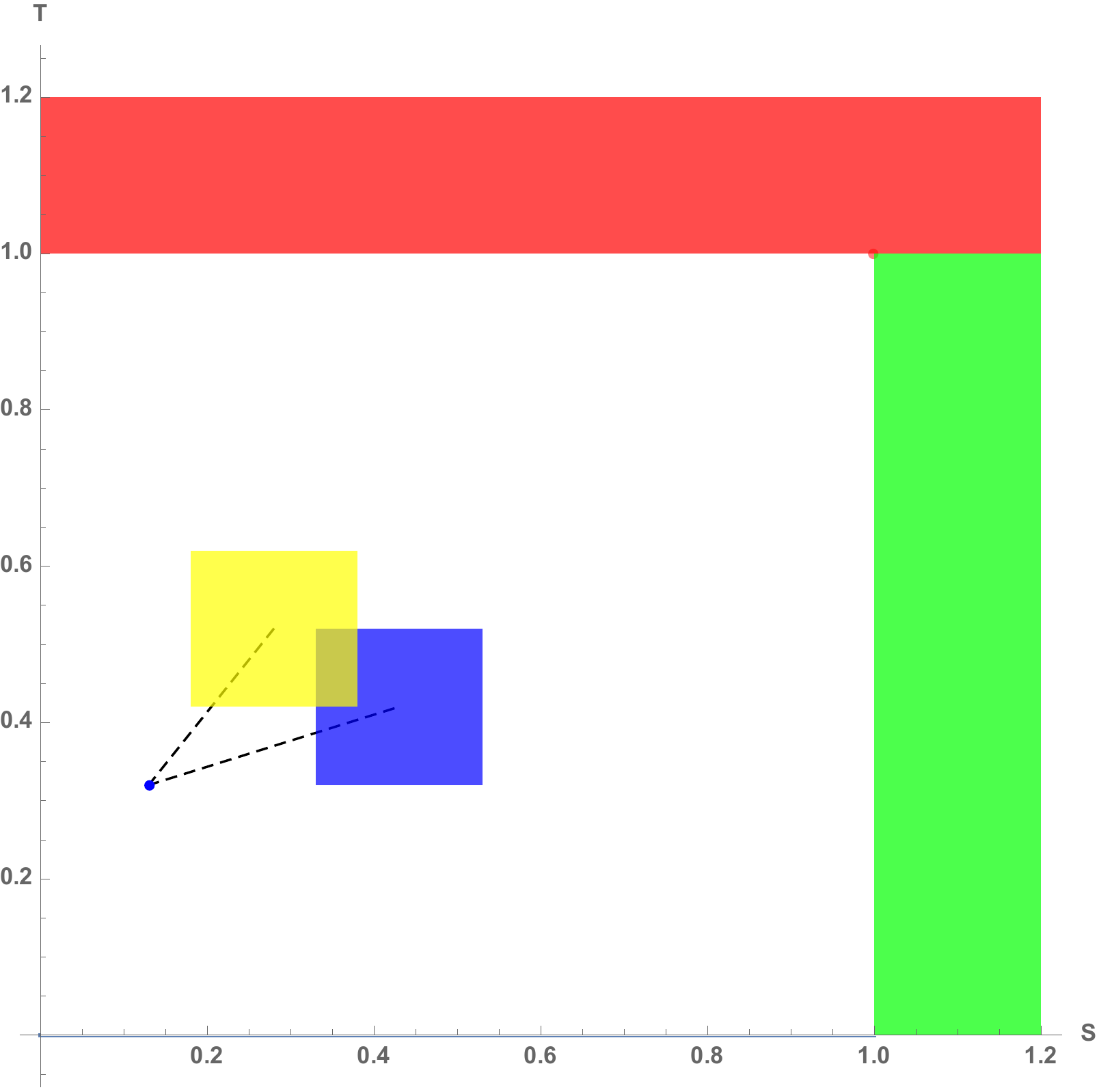}
\hspace{10mm}
 \includegraphics[scale=0.3]{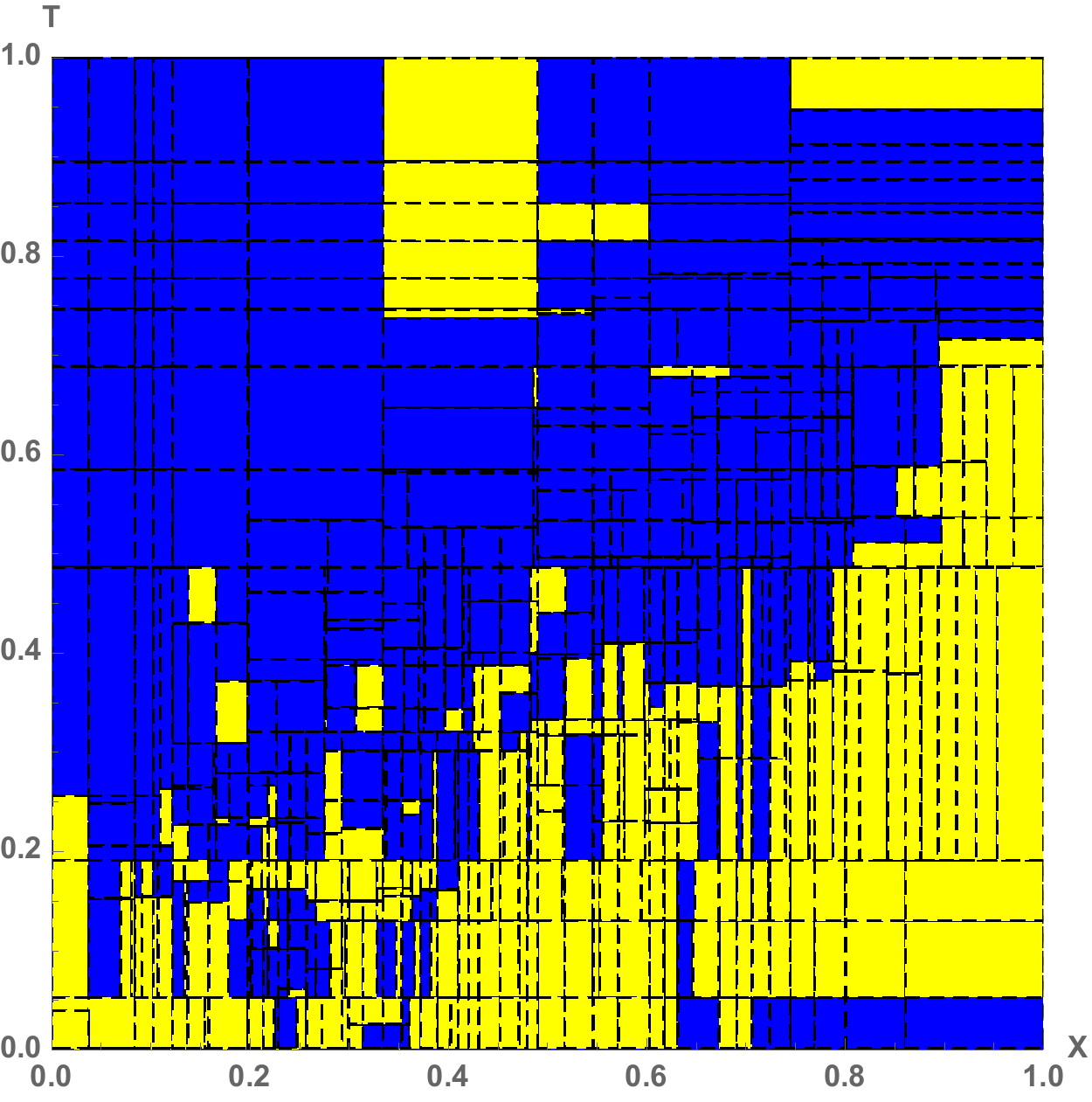} 
    \caption{Left: a Semi-Random Walk on $\states=[0,1.2]\times [0,1.2]$. Green: goal area, red: failure area,
      blue dot: current state, yellow/blue squares: successor state squares for \emph{fast} (blue) and \emph{slow} (yellow)
      actions. Right: partition of $[0,1]\times [0,1]$ and strategy learned by UPPAAL STRATEGO; partition regions colored
      according to actions prescribed by the strategy.}
  \label{randomwalkfig}
\end{figure}

In  \cite{jaeger2019teaching},  we  proposed   two  reinforcement  learning  algorithms
implemented in UPPAAL STRATEGO \cite{David-et-al2015}, using online partition
refinement techniques. In that work we experimentally demonstrated its
improved  convergence  tendencies  on  a range  of  models.   For  the
semi-random  walk example, the  online  learning
algorithm   returns  the  strategy illustrated  on the right of 
Fig.~\ref{randomwalkfig}.
%

%
However,  despite  its   efficiency  and  experimentally  demonstrated
convergence  properties,   the  learning  approach   of  \cite{jaeger2019teaching}
provides no  hard guarantees as to  how far away the  expected cost of
the  learned strategy  is  from the  optimal one.   In  this paper  we
propose  a   step-wise  partition   refinement  process,   where  each
partitioning  induces  a  finite-state   imprecise  MDP  (IMDP).
From the induced IMDP we can derive upper and lower bounds on the
expected cost of the  original infinite-state  Euclidean MDP.
As a crucial result,  we prove the correctness of these bounds, i.e., that they
are always guaranteed to contain the true
expected  cost. Also, we provide   value iteration procedures for
computing   lower  and   upper  expected   costs  of   IMDPs.
Figure~\ref{partstrat} shows upper and lower bounds on the expected cost
over the regions shown in  Figure~\ref{randomwalkfig}.

\begin{figure}
\centering
  \includegraphics[scale=0.4]{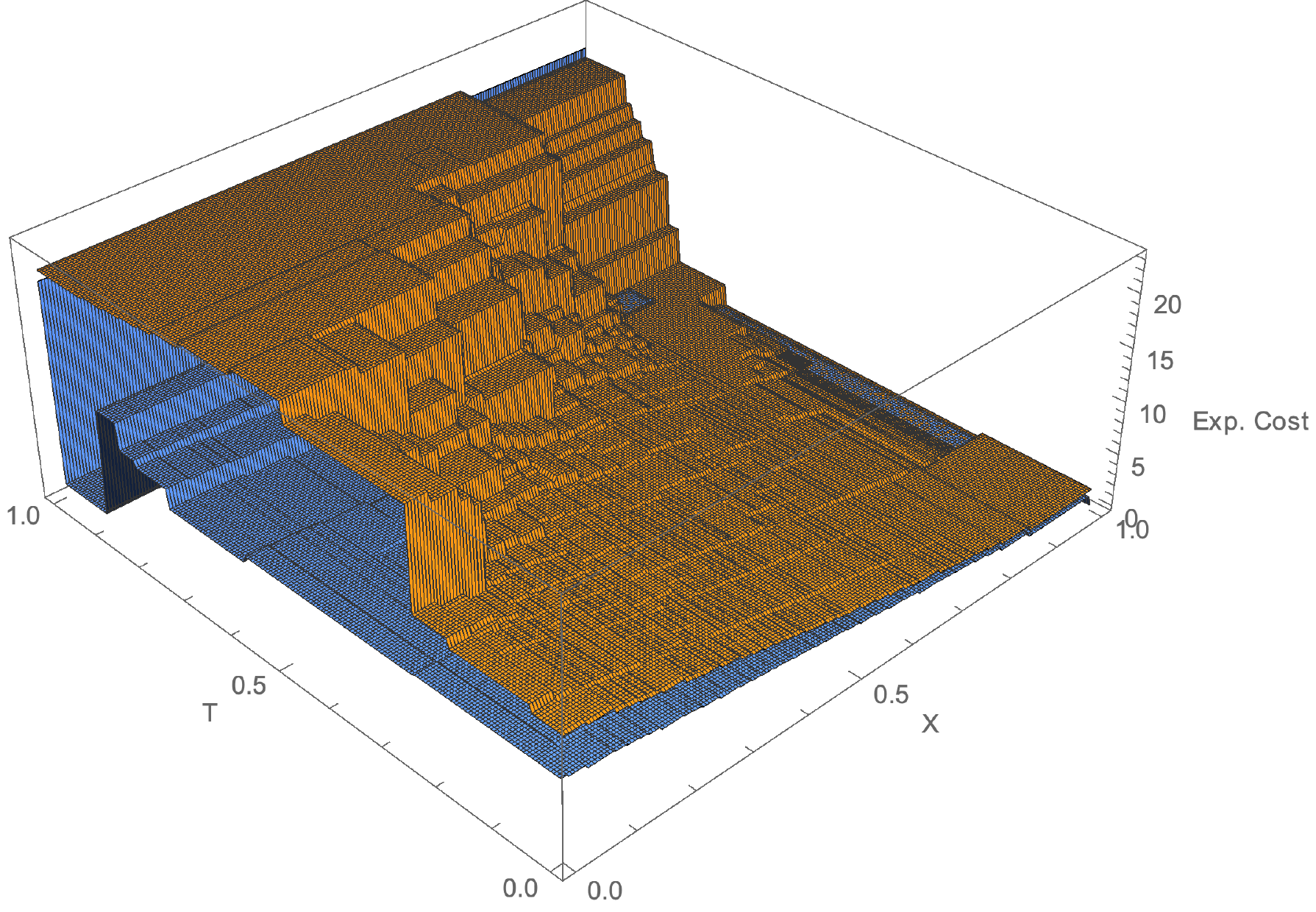} 
  \caption{Lower and upper cost bounds for the learned partition.}
  \label{partstrat}
\end{figure}

Applying the IMDP  value iteration procedures to the partition learned
by  UPPAAL STRATEGO therefore allows us to compute
guaranteed lower  and upper bounds  on the  expected cost, and thereby
validate the results of reinforcement learning. 
The main contributions of this paper can by summarized as follows:
\begin{itemize}
\item We define IMDP abstractions of infinite state Euclidean MDPs, and establish
  as key theoretical properties: the correctness of value iteration to compute
  upper and lower expected cost functions, the correctness of the upper and lower
  cost functions as bounds on the cost function of the original Euclidean MDP, and,
  under a restriction to finite time horizons, the convergence of upper and lower
  bounds to the actual cost values.
\item  We demonstrate the applicability of the general framework to analyze the
  accuracy of strategies learned by reinforcement learning.
\end{itemize}


\paragraph{Related Work.}

Our work is closely related to various types of MDP models proposed in different
areas. 
\emph{Imprecise Markov Chains and Imprecise Markov Decision processes} have been
considered in areas such as operations research and
artificial intelligence~\cite{white1994markov,crossman2009imprecise,troffaes2015using}.
The focus here typically is on approximating optimal policies for fixed, finite state
spaces.  In the same spirit,  but from a
verification point of view, \cite{chen2013complexity} focuses on reachability probabilities.

\emph{Lumped Markov chains} are obtained by aggregating sets of states of a Markov Chain into a single
state. Much work is devoted to the question of when and how the resulting process
again is a Markov chain (it rarely is)~\cite{rubino1991finite,derisavi2003optimal}.
The interplay of  lumping and imprecision is considered in~\cite{erreygers2018computing}
Most work in this area is concerned with finite state spaces.
Abstraction by state space partitioning (lumping) can be understood as a special form of
\emph{partial observability} (one only observes which  partition element the current state belongs to).
A combination or partial observability with imprecise probabilities is considered
in~\cite{itoh2007partially}

\cite{DBLP:conf/qest/KwiatkowskaNP06} introduce \emph{abstractions of finite state MDPs} by partitioning
  the state space. Upper and lower
  bounds for reachability probabilities are obtained from the abstract MDP, which is formalized as
  a two player stochastic game. \cite{DBLP:conf/adhs/LunWDA18} is  concerned with obtaining
  accurate specifications of an abstraction obtained by state space partitioning. The underlying
  state space is
  finite, and a fixed partition is given.  

  Thus, while there is a large amount of closely related work on abstracting MDPs by state space partitioning, and
  imprecise MDPs that can result from such an abstraction, to the best of our knowledge, our work is distinguished from
  previous work by: the consideration of infinite continuous state spaces for the underlying models of primary interest,
  and the focus on the properties of refinement sequences induced by partitions of increasing granularity.

\section{Euclidean MDP and Expected Cost}
\begin{definition}[Euclidean Markov Decision Processes]\label{def:mdp}
  A      Euclidean Markov decision process        (EMDP)
  is              a              tuple
  $\mdp=(\states, \goal, \act, \transitions, \cost)$ where:
\begin{itemize}
\item $\states \subseteq  \reals^K$ is a
  measurable subset of the $K$-dimensional Euclidean space equipped with the Borel $\sigma$-algebra
  $\borel^K$. 
\item $\goal\subseteq\states$ is a measurable set of goal states,
\item $\act$ is a finite set of actions,
\item    $T:\states    \times   \act  \times\borel^K \rightarrow[0,1]$ defines for every $a\in\act$ a
  transition kernel on $(\states,\borel^K)$,
  i.e., $T(s,a,\cdot)$ is a probability distribution on $\borel^K$ for all
  $s\in\states$, and $T(\cdot,a,B)$ is measurable for all $B\in\borel^K$.
  Furthermore, the set of goal states is absorbing, i.e. for all $s\in\goal$ and all $a\in\act$:
  $T(s,a,\goal)=1$.
\item  $\cost:\states\times\act\rightarrow\reals_{\geq  0}$   is  a
  cost-function for state-action pairs, such that for all $a\in\act$:
  $\cost(\cdot,a)$ is measurable, and
  $\cost(s,a)=0$ for all $s\in\goal$.
\end{itemize}
\end{definition}

A run $\run$ of an MDP is a sequence of alternating states and actions
$s_1a_1s_2a_2\cdots$. 
We denote the set of all  runs of an EMDP
$\mdp$   as  $\runs_\mdp$.
We use $\run_i$ to denote $(s_i, a_i)$, $\run_{\leq i}$ for the prefix
$s_{1}a_{1}s_{2}a_{2}\cdots s_{i}a_{i} $,  and 
 $\run_{>i}$  for the  tail  $s_{i+1}a_{i+1}s_{i+2}a_{i+2}\cdots $ of a run.
The \emph{cost} of a run is
\begin{equation*}
\costinfty(\run) := \sup_{N} \sum_{i = 1}^N \cost(\pi_i) \in[0,\infty] \,.
\end{equation*}

The set $\runs_\mdp$ is equipped with the product $\sigma$-algebra
$(\borel^K\otimes 2^{\act})^{\infty}$ generated by the cylinder sets
$B_1\times\{a_1\}\times \cdots \times B_n\times\{a_n\}\times(\states\times\act)^{\infty}$ ($n\geq 1$, $B_i\in\borel^K$, $a_i\in\act$).
We denote with $\borel_+$ the Borel $\sigma$-algebra restricted to the non-negative reals, and
with $\bar{\borel}_+$ the standard extension  to
  $\bar{\reals}_{\geq 0}:=\reals_{\geq 0}\cup\{\infty\}$, i.e. the sets of the form $B$ and $B\cup\{\infty\}$, where
   $B\in\borel_+$. 
  \begin{restatable}{lemma}{measurable}
    $\costinfty$ is $(\borel^K\otimes 2^{\act})^{\infty} - \bar{\borel}_+$ measurable.
  \end{restatable}

  Due to space constraints  proofs  are only included in the extended online version of this paper.

  We next consider strategies for EMDPs. We limit ourselves to memoryless and stationary strategies, noting
  that on the rich Euclidean state space $\states$ this is less of a limitation than on finite state spaces,
  since a non-stationary, time dependent strategy can here be turned into a stationary strategy by adding
  one real-valued dimension representing time. 
  
\begin{definition}[Strategy]
  A (memoryless,stationary) strategy for an MDP $\mdp$ is a function
  $\sigma:\states\rightarrow(\act\rightarrow[0,1])$, mapping states to  probability
  distributions over $\act$, such that for every $a\in\act$ the function
  $s\in\states\mapsto \sigma(s)(a)$ is measurable.
\end{definition}

The following lemma is mostly a technicality that needs to be established in order to
ensure that an MDP in conjunction with a strategy and an initial state distribution
 defines a Markov process on $\states\times\act$, and
hence a probability distribution on $\runs_\mdp$.
\begin{restatable}{lemma}{strategykernel}
  \label{lem:strategykernel}
  If $\sigma$\ is a strategy, then
  \begin{equation}
    \label{eq:Tsigma}
    \begin{array}{lll}
      T_\sigma: & (\states\times\act)\times (\borel^K\times 2^{\act})    & \rightarrow [0,1] \\
      & ( (s,a),(B,A)  ) & \mapsto \int_B \sigma(s')(A) T(s,a,ds')
    \end{array}
  \end{equation}
  is a transition kernel on $(\states\times\act,\borel^K\times 2^{\act})$. 
\end{restatable}


Usually, an initial state distribution
will be given by a fixed initial state $s=s_1$. We then denote the resulting distribution over
$\runs_\mdp$ by  $\pruns{s,\sigma}$ (this also depends on the underlying $\mdp$; to avoid
notational clutter, we do not always make this dependence explicitly in the notation). 

\begin{definition}[Expected Cost]
  Let $s\in\states$. The expected cost at $s$ under strategy $\sigma$ is the expectation
  of $\costinfty$ under the distribution  $\pruns{s,\sigma}$, denoted
  $\expected_{\sigma}(\cost, s)$. The expected cost at initial state $s$ then is
  defined as
  \begin{equation*}
    \expected(\cost, s) :=\inf_{\sigma}\, \expected_{\sigma}(\cost, s) \in [0,\infty] \,.
  \end{equation*}
\end{definition}

\begin{example}
  If $s\in\goal$, then for any strategy $\sigma$: $\pruns{s,\sigma}(\bigcap_{i\geq 1}\{s_i\in\goal\}  )=1$,
  and hence $\expected(\cost, s) =0$. However, $\expected(\cost, s) = 0$ can also hold for
  $s \not\in\goal$, since $\cost(s,a)=0$ also is allowed for non-goal states $s$. 
\end{example}

Note that, for any strategy $\sigma$, the functions $\expected_\sigma(\cost, \cdot)$ and 
$\expected(\cost, \cdot)$ are $[0,\infty]$-valued measurable functions on $\states$. This follows by measurability of $\cost(\cdot,a)$ and $\sigma(\cdot)(a)$, for all $a \in\act$, and \cite[Theorem 13.4]{Billingsley1986}.

\subsection{Value Iteration for EMDPs}
\label{sec:vi4EMDPs}

We next show that expected costs in EMDPs can be computed by value iteration. Our results are
closely related to Theorem 7.3.10 in~\cite{Puterman2005}. However, our scenario differs from the one
treated by Puterman~\citep{Puterman2005} in that we deal with uncountable state spaces, and in that
we want to permit infinite cost values.
Adapting Puterman's notation~\cite{Puterman2005}, we introduce two operators, 
$\viop$ and $\viop^\sigma$, on $[0,\infty]$-valued measurable functions $E$ on $\states$, defined 
as follows:
\begin{align*}
  \label{eq:Loperator}
  \viop E(s) &:=  \min_{a\in\act} \left( \cost(s,a) +
  \int_{t \in \states} E(t) \,  T(s,a, \dee t)  \right) \,,
  \\
  \viop^\sigma E(s) &:=   \sum_{a \in \act} \sigma(s)(a) \cdot \left( \cost(s,a) + \int_{t \in \states} E(t) \,  T(s,a, \dee t)  \right) \,,
\end{align*}

The operators above are well-defined:
\begin{restatable}{lemma}{LEmeasurable}
\label{lem:LEmeasurable}  
If $E$ is measurable, so are $\viop E$ and $\viop^\sigma E$.
\end{restatable}

The set of $[0,\infty]$-valued measurable functions on $\states$ forms a complete partial order
under the point wise order $E \leq E'$ iff $E(s) \leq E'(s)$, for all $s \in \states$. The top $\top$ and 
bottom $\bot$ are respectively given by the constant functions $\top(s) := \infty$,  $\bot(s) := 0$, for $s \in \states$. Meet and join are the point-wise infimum and point-wise supremum,  respectively. 
By their definition, it is easy to see that both $\viop$ and $\viop^\sigma$ are monotone operators.

Since the set of actions $\act$ is finite, for every $E$ we can define a deterministic strategy $d$, 
such that $\viop E= \viop^d E$. We can establish an even stronger relation:
\begin{restatable}{lemma}{deteministicstrategies}
  \label{lem:deteministicstrategies}
$\inf_\sigma \viop^\sigma = \viop$.
\end{restatable}

As a first main step we can show that 
the expected cost under the strategy $\sigma$ is a fixed point for the operator $\viop^\sigma$:
\begin{restatable}{proposition}{volterraStrategy}
  \label{prop:volterraStrategy}
For any strategy $\sigma$, $\expected_\sigma(\cost, \cdot) = \viop^\sigma\expected_\sigma(\cost, \cdot)$.
\end{restatable}

As a corollary of Lemma~\ref{lem:deteministicstrategies} and Proposition~\ref{prop:volterraStrategy}, $\expected(\cost, \cdot)$ is a pre-fixpoint of the $\viop$ operator. 
Moreover, we can show that it is the least pre-fixpoint of $\viop$.


\begin{restatable}{proposition}{leastprefixpoint}
\label{prop:leastprefixpoint}
$\expected(\cost, \cdot) \geq \viop \expected(\cost, \cdot)$. Moreover, 
if $E \geq \viop E$, then $E \geq \expected(\cost, \cdot)$.
\end{restatable}

By Proposition~\ref{prop:leastprefixpoint} and Tarski fixed point 
theorem,  $\expected(\cost, \cdot)$ is the least fixed point of $\viop$. 
The following theorem, provides us with a stronger result, namely, that $\expected(\cost, \cdot)$  
is the supremum of the point-wise increasing chain
\begin{equation*}
  \bot \leq \viop \bot \leq \viop^2 \bot \leq \viop^3 \bot \leq \dots
\end{equation*}

We denote
\begin{equation}
  \label{eq:defLk}
  L^n:= \viop^n \bot\ (n\geq 1),\ \ \ \mbox{and}\ \ L:=\sup_{n \geq 0}  L^n
\end{equation}
The following theorem then states that value iteration converges to $\expected(\cost,\cdot)$.

\begin{restatable}{theorem}{policyiter}
  \label{theo:policyiter}
$\expected(\cost,\cdot) = L$.
\end{restatable}

\section{Imprecise MDP}

The value iteration of Theorem~\ref{theo:policyiter} is a mathematical process, not an
algorithmic one, as it is defined pointwise on the uncountable state space $\states$. 
Our goal, therefore, is to approximate the expected cost function $\expected(\cost,\cdot)$
of an EMDP by expected
cost functions on finite state spaces consisting of partitions of $\states$. In order to
retain sufficient information of the original EMDP to be able to derive provable upper and
lower bounds for $\expected(\cost,\cdot)$, we approximate the EMDP by an
\emph{Imprecise Markov Decision Processes (IMDPs)}~\cite{white1994markov}.

\begin{definition}[Imprecise Markov Decision Processes]
  \label{def:imdp}
  A finite state, imprecise Markov decision process (IMDP) is a tuple
  $\mdp=(\states, \goal, \act, \transitions^*, \cost^*)$ where:
\begin{itemize}
\item $\states$ is a finite set of states
\item $\goal\subseteq\states$ is the set of goal states,
\item $\act$ is a finite set of actions,
\item    $T^*:\states    \times   \act    \rightarrow 2^{(\states\rightarrow\reals_{\geq  0})}$ assigns to
  state-action pairs a closed
  set of probability distributions   over  $\states$;
  the set of goal states is absorbing, i.e., for all
  $s\in\goal$ and all $T(s,a)\in T^*(s,a)$: $\sum_{t\in\goal}T(s,a)(t) =1$,
\item  $\cost^*:\states\times\act\rightarrow 2^{\reals_{\geq  0}}$   assigns to
  state-action pairs a closed set of costs, such that for all $s\in\goal, a\in\act$:
  $\cost^*(s,a)= \{0\}$.
\end{itemize}
\end{definition}


Memoryless, stationary strategies $\sigma$ are defined as before. In order to turn an IMDP into a fully probabilistic
model, one also needs to resolve the choice of a transition probability distribution and cost value. 

\begin{definition}[Adversary, Lower/Upper expected cost]
\label{def:adversary}  
An \emph{adversary} $\alpha$ for an IMDP consists of two functions
\begin{displaymath}
  \begin{array}{llll}
    \alpha_T: & (s,a) & \mapsto \alpha_T(s,a)\in T^*(s,a) & ((s,a)\in\states\times\act), \\
    \alpha_C: & (s,a) & \mapsto \alpha_C(s,a)\in \cost^*(s,a) & ((s,a)\in\states\times\act).
  \end{array}
\end{displaymath}
A strategy $\sigma$, an adversary $\alpha$, and   an \emph{initial state} $s$ together
define a probability distribution
  $\pruns{s,\sigma,\alpha}$ over runs $\pi$ with $s_1=s$, and hence the expected cost
  $\expected_{\sigma,\alpha}(\cost^*(\pi),s)$. We then define the lower and upper expected cost as
    \begin{eqnarray}
      \expected^{\min}(\cost^*(\pi),s)  & :=  & \min_{\sigma} \min_{\alpha} \expected_{\sigma,\alpha}(\cost^*(\pi),s)
      \label{eq:Emin}\\
      \expected^{\max}(\cost^*(\pi),s)  & :=  & \min_{\sigma} \max_{\alpha} \expected_{\sigma,\alpha}(\cost^*(\pi),s) 
      \label{eq:Emax}
    \end{eqnarray}
\end{definition}

Since $T^*(s,a)$ and $\cost^*(s,a)$ are required to be closed sets, we can here write
$\min_{\alpha}$ and $\max_{\alpha}$ rather than $\inf_{\alpha}$, $\sup_{\alpha}$. Furthermore, the closure conditions
are needed to justify a restriction to stationary adversaries, as the following example shows (cf. also Example 7.3.2
in~\cite{Puterman2005}).

\begin{example}
  Let $\states=\{s_1,s_2,s_3\}$, $\act=\{a\}$, We write $(p_1,p_2,p_3)$ for a transition probability distribution $T$ with
  $T(s_i)=p_i$. Then let $T^*(s_1,a)=\{(p_1,p_2,p_3): p_1\in ]0,1[, p_2=1-p_1\}$,
  $T^*(s_2,a)=T^*(s_3,a)=\{(0,0,1)\}$. 
  $\cost^*(s_1,a)=\cost^*(s_3,a)=\{0\}$,  $\cost^*(s_2,a)=\{1\}$. Since there is only one action, there is only one strategy
  $\sigma$. For $i\geq 1$ let $\epsilon_i\in ]0,1[$ such that $\prod_{i=1}^{\infty}\epsilon_i = \delta >0$. Then, if the adversary at the
  $i$'th step selects transition probabilities $(\epsilon_i,1-\epsilon_i,0)$ one obtains
  $\expected^{\min}(\cost^*(\pi),s_1)=1-\delta$. For every stationary adversary the transition from $s_1$ to $s_2$ will be
  taken eventually with probability 1, so that here $\expected^{\min}(\cost^*(\pi),s_1)=1$.
\end{example}


We note that only in the case of $\expected^{\max}$ does $\alpha$ act as an ``adversary'' to the
strategy $\sigma$. In the case of $\expected^{\min}$, $\sigma$ and $\alpha$ represent co-operative
strategies. In other definitions of imprecise MDPs only the transition probabilities are
set-valued~\cite{white1994markov}. Here we also allow an imprecise cost function. Note, however, that
for the definition of $\expected^{\min}(\cost^*,s)$ and $\expected^{\max}(\cost^*,s)$ the adversary's
strategy $\alpha_C$ will simply be to select the minimal (respectively maximal) possible costs, and that
we can also obtain $\expected^{\min}, \expected^{\max}$ as the expected lower/upper costs on
IMDPs with point-valued cost functions
\begin{displaymath}
  \begin{array}{lll}
    \cost^{\min}(s,a)& := & \min \cost^*(s,a),\\
     \cost^{\max}(s,a)& := & \max \cost^*(s,a),
  \end{array}
\end{displaymath}
where then the adversary has no choice for the strategy $\alpha_C$.

\subsection{Value Iteration for IMDPs}

We now characterize $\expected^{\min},\expected^{\max}$ as limits of  value iteration,
again following the strategy of the proof of Theorem 7.3.10 of~\cite{Puterman2005}.
In this case, the proof has to be adapted to accommodate the additional optimization of the
adversary, and, as in Section~\ref{sec:vi4EMDPs}, to allow for infinite costs.
We again start by defining suitable operators $\viop^{\min}, \viop^{\max}$ on
$[0,\infty]$-valued  functions $C$ defined on $\states$:

\begin{equation}
  \label{eq:Loperator}
  (\viop^{\minmax} C)(s) :=  \min_{a\in\act}\left( \cost^{\minmax}(s,a)+
       \minmax_{T\in T^*(s,a)}\sum_{s'}T(s') C(s')  \right), 
   \end{equation}
   where $\minmax\in\{\min,\max\}$.
   The mapping
   \begin{equation}
     \label{eq:alphaTC}
    \alpha_T^{\minmax}(C): (s,a)\mapsto \arg\!\!\!\!\!\minmax_{T\in T^*(s,a)}\sum_{s'}T(s') C(s')
   \end{equation}
   defines the $\alpha_T$ of an adversary. Similarly
   \begin{equation}
     \label{eq:sigmaC}
     \sigma^{\minmax}(C): s\mapsto \arg\min_{a\in\act}\left( \cost^{\minmax}(s,a)+
       \sum_{s'} \alpha_T^{\minmax}(C)(s,a)C(s')  \right)
   \end{equation}
   defines a strategy.

   Let $\bot$ be the function that is constant 0 on $\states$. Denote
   \begin{equation}
     \label{eq:defLoptn}
     L^{\minmax,n}:=(\viop^{\minmax})^n \bot, \ \ \ \mbox{and}\ \ L^{\minmax}:= \sup_{n \geq 0} L^{\minmax,n}
   \end{equation}

We can now state the applicability of value iteration for IMDPs as follows:

   \begin{restatable}{theorem}{viimdp}
     \label{theo:viimdp}
  Let $\minmax\in\{\min,\max\}$. Then
   \begin{equation}
     \label{eq:limvalit}
      \expected^{\minmax}(\cost^*(\pi),\cdot)=L^{\minmax}
   \end{equation}
 \end{restatable}

 We note that even though $\viop^{\minmax}$, in contrast to the $\viop$ operator for EMDPs, now
 only needs to be computed over a finite state space, we  do not obtain from
 Theorem~\ref{theo:viimdp} a fully specified algorithmic
 procedure for the computation of $\expected^{\minmax}$, because the optimization over $T^*(s,a)$ contained in (\ref{eq:Loperator}) will
 require customized solutions that depend on the structure of the $T^*(s,a)$.

\section{Approximation by Partitioning}\label{sec:approx_partition}

From now on we only consider EMDPs whose state space $\states$ is a compact
subset of $\reals^K$. 
We approximate such a Euclidean MDP by  IMDPs constructed from finite partitions
of  $\states$.
In the following, we denote with
$\partition = \{\nu_1,\ldots,\nu_{|\partition|}\}\subset  2^\states$ a finite
partition of $\states$.
We  call an element $\nu\in\partition$ a
\emph{region}  and  shall  assume  that   each  such  $\nu$  is  Borel
measurable.  For $s\in\states$ we
denote by  $[s]_\partition$ the unique region  $\nu\in\partition$ such
that  $s\in\nu$.
The \emph{diameter} of a region is $\delta(\nu):=\sup_{s,s'\in\nu}\parallel s-s'\parallel$, and
 the \emph{granularity} of a
$\partition$ is defined as  
$\delta(\partition):= \max_{\nu\in\partition} \delta(\nu) $.
We say that a partition $\partitionB$ refines
a partition $\partition$  if for any $\nu\in\partitionB$ there exist
$\mu\in\partition$     with     $\nu\subseteq\mu$.     We     write
$\partition\sqsubseteq\partitionB$ in this case.

A Euclidean MDP $\mdp= (\states, \goal, \act, \transitions, \cost)$ and
a partition $\partition$ of  $\states$ induces an
abstracting    IMDP
\cite{DBLP:conf/qest/KwiatkowskaNP06,DBLP:conf/adhs/LunWDA18} according to the following
definition. 

\begin{definition}[Induced IMDP]
  \label{def:inducedimdp}
  Let $\mdp=(\states,  \act, \initial, \transitions,  \cost,\goal)$ be
  an  MDP, and  let $\partition$  be a  finite partition  of $\states$
  consistent  with  $\goal$ in the sense that  for any $\nu\in\partition$
  either $\nu\subseteq\goal$ or  $\nu\cap\goal=\emptyset$.
  The IMDP defined by $\mdp$ and $\partition$ then is 
  $\mdp_\partition=(\partition,\goal_\partition,\act, T_\partition^*,\cost_\partition^*)$, where
  
  \begin{itemize}
  \item $\goal_\partition=\{\nu\in\partition | \nu\subseteq\goal\}$
    \item  \begin{displaymath}
    T_\partition^*(\nu,a)=\closure\{ T_\partition(s,a) \mid s\in\nu \} ,
  \end{displaymath}
  where $ T_\partition(s,a)$ is the marginal
  of $T(s,a,\cdot)$ on $\partition$, i.e. $T_\partition(s,a)(\nu')=\int_{\nu'}T(s,a,dt)$,
  and $\closure$ denotes topological closure.
  \item 
    \begin{displaymath}
      \cost_\partition^*(\nu,a) = \closure(\{ C(s,a)| s\in\nu\})
    \end{displaymath}
  \end{itemize}
 
\end{definition}

The following theorem states how an induced IMDP approximates the underlying
Euclidean MDP. In the following, we use sub-scripts on expectation operators to
identify the (I)MDPs that define the expectations. 

\begin{restatable}{theorem}{correctbounds}
\label{theo:correctbounds}
  Let $\mdp$ and  $\partition$ as in Definition~\ref{def:inducedimdp}.
  Then for all $s\in\states$:
  \begin{equation}
    \label{eq:bounds}
    \expected_{\mdp_\partition}^{\min}(\cost_\partition^*,[s]_\partition)\leq
     \expected_\mdp(\cost,s) \leq
     \expected_{\mdp_\partition}^{\max}(\cost_\partition^*,[s]_\partition).
  \end{equation}
   If  $\partition\sqsubseteq\partitionB$, then
   $\partitionB$  improves the bounds in the sense that
     \begin{eqnarray}
       \expected_{\mdp_\partition}^{\min}(\cost_\partition^*,[s]_\partition) & \leq &
       \expected_{\mdp_\partitionB}^{\min}(\cost_\partition^*,[s]_\partitionB), \label{eq:AB1}\\
       \expected_{\mdp_\partition}^{\max}(\cost_\partition^*,[s]_\partition) & \geq &
 \expected_{\mdp_\partitionB}^{\max}(\cost_\partitionB^*,[s]_\partitionB). \label{eq:AB2}
     \end{eqnarray}
 \end{restatable}

 \newcommand{\dtv}{d_{\emph{tv}}}
\newcommand{\PNtau}[2]{P^{#1}_{\tau,\alpha^{#2}}}
\newcommand{\ENtau}[2]{\expected ^{#1}_{\tau,\alpha^{#2}}}

 Our goal now is to establish conditions under which the approximation
 (\ref{eq:bounds}) becomes arbitrarily tight for partitions of sufficiently
 high granularity. This will require certain continuity conditions for
 $\mdp$ as spelled out in the following definition. In the following, $\dtv$ stands for
 the total variation distance between distributions. Note that we will be using
 $\dtv$ both for discrete distributions on partitions $\partition$, and for continuous
 distributions on $\states$.

 \begin{definition}[Continuous Euclidean MDP]
   \label{def:continuous}
   A Euclidean MDP $\mdp$ is \emph{continuous} if
   \begin{itemize}
   \item For each $\epsilon>0$ there exists $\delta>0$, such that: for all partitions $\partition$,
     if $\delta(\partition)\leq \delta$, then for all $\nu\in\partition$, $s,s'\in\nu$, $a\in\act$:
     $\dtv(T(s,a),T(s',a))\leq \epsilon$.
     \item $\cost$ is continuous on $\states$ for all $a\in\act$.
   \end{itemize}
 \end{definition}
 
 We observe that due to the assumed compactness of $\states$,
 the first condition of Definition~\ref{def:continuous} is satisfied if
 $T$ is defined as a function $T(s,a,t)$ on $\states\times\act\times\states$ that for 
 each $a$ as a function  of $s,t$ is continuous on  $\states\times\states$, and such that 
 $T(s,a,\cdot)$ is for all $s,a$ a density function relative to Lebesgue measure. 

 We next introduce some notation for $N$-step expectations and distributions. In the
 following, we use $\tau$ to denote strategies for induced IMDPs defined on partitions $\partition$,
 whereas $\sigma$ is reserved for strategies defined on Euclidean state spaces $\states$.
 For a given partition $\partition$ and strategy $\tau$ for $\mdp_\partition$ let
 $\alpha^+,\alpha^-$ denote two strategies for the adversary (to be interpreted as strategies
 that are close to  achieving $sup_{\alpha} \expected_{\tau,\alpha}(\cost^*(\pi),\cdot)$ and
 $inf_{\alpha} \expected_{\tau,\alpha}(\cost^*(\pi),\cdot)$, respectively, even though we will not
 explicitly require properties that derive from this interpretation). We then denote 
 with $\PNtau{N}{+},\PNtau{N}{-}$ the distributions defined by $\tau,\alpha^+$ and $\tau,\alpha^-$
 on run prefixes of length $N$, and with $\ENtau{N}{+},\ENtau{N}{-}$ the corresponding expectations
 for the sum of the first $N$ costs $\sum_{i=1}^N \alpha_C^{+[-]}(\nu_i,a_i)$. The $P^N$ and $\expected^N$ also
 depend on the initial state $\nu_1$. To avoid notational clutter, we do not make this explicit in
 the notation. 
 We then obtain the following approximation guarantee:
  
 \begin{restatable}{theorem}{mainlem}
   \label{theo:mainlem}
   Let $\mdp$ be a continuous EMDP. 
   For all $N$, $\epsilon>0$ there exists $\delta>0$, such that for all partitions $\partition$ with
   $\delta(\partition)\leq \delta$, and all strategies $\tau$ defined on $\partition$:
   \begin{equation}
     \label{eq:ENeq}
     |\ENtau{N}{+}- \ENtau{N}{-}| \leq \epsilon
   \end{equation}
   and
   \begin{equation}
     \label{eq:PNeq}
     \dtv(\PNtau{N}{+},\PNtau{N}{-}) \leq \epsilon.
   \end{equation}
 \end{restatable}

 Theorem~\ref{theo:mainlem} is a strengthening of Theorem 2 in \cite{jaeger2019teaching}. The latter applied
 to processes that are guaranteed to terminate within $N$ steps. Our new theorem applies to the expected cost
 of the first $N$ steps in a process of unbounded length. When the process has a bounded time horizon of no more
 than $N$ steps, and if we let $\tau,\alpha^+,\alpha^-$ be the strategy and the adversaries that achieve the
 optima in (\ref{eq:Emin}), respectively (\ref{eq:Emax}), then (\ref{eq:ENeq}) becomes
 \begin{equation}
   \label{eq:convergenceNbounded}
   | \expected^{\max}_{\mdp_{\partition}}-\expected^{\min}_{\mdp_{\partition}} |\leq \epsilon.
 \end{equation}

 We conjecture that this actually also holds true for arbitrary EMDPs:
 
 \begin{conjecture}\label{thm:convergence}
   Let $\mdp$ be 
   a continuous Euclidean MDP. Let 
   $\partition_0\sqsubseteq\partition_1\sqsubseteq
   \cdots\sqsubseteq  \partition_i\sqsubseteq\cdots$   be  a 
   sequence of partitions consistent  with $\goal$ such that
   $\underset{i\rightarrow\infty}{\lim}\delta(\partition_i)=0$.  Then for
   all $s\in\states$:
  \[
     \underset{i\rightarrow\infty}{\lim}
     \expected_{\mdp_{\partition_i}}^{\min}(\cost^*_{\partition_i},[s]_{\partition_i})
       =
       \expected_{\mdp}(\cost,s)
       =
       \underset{i\rightarrow\infty}{\lim}
     \expected_{\mdp_{\partition_i}}^{\max}(\cost^*_{\partition_i},[s]_{\partition_i}).
   \]
 
   \end{conjecture}

   The approximation guarantees given by Theorems ~\ref{theo:correctbounds} and
   \ref{theo:mainlem} have two important implications: first, they guarantee the correctness
   and asymptotic accuracy of upper/lower bounds computed by value iteration in IMDP
   abstractions of the underlying EMDP. Second, they show that the hypothesis space
   of strategies defined over finite partitions that underlies the reinforcement learning
   approach of~\cite{jaeger2019teaching} is adequate in the sense that it contains strategy
   representations that approximate the optimal strategy for the underlying continuous
   domain arbitrarily well.

\section{Examples and Experiments}

We now use our semi-random walker example to illustrate the theory presented in
the preceding sections, and to demonstrate
its applicability to the validation of machine learning models.

\subsection{IMDP Value Iteration}

We first illustrate experimentally the bounds and convergence properties expressed by
Theorems~\ref{theo:correctbounds} and~\ref{theo:mainlem}.
For this we consider a nested sequence of partitions of the continuous state space
$\states=[0,x_{max}]\times [0,t_{max}]$ consisting of regular grid partitions $\partition=\partition(\Delta)$
defined by a width parameter $\Delta$ for the regions. We run value iteration to compute
$\expected_{\mdp_\partition(\Delta)}^{\min}$ and $\expected_{\mdp_\partition(\Delta)}^{\max}$ for
the values $\Delta \in \{ 0.1, 0.05, 0.025 \}$.
For illustration purposes, we plot expected cost functions along 
one-dimensional sections  $\states'_t = [0,x_{max}] \times \{t\}$ for the two
fixed time points $t=0$ and $t= 0.7$.

\begin{figure}
  \centering
   \includegraphics[scale=0.47]{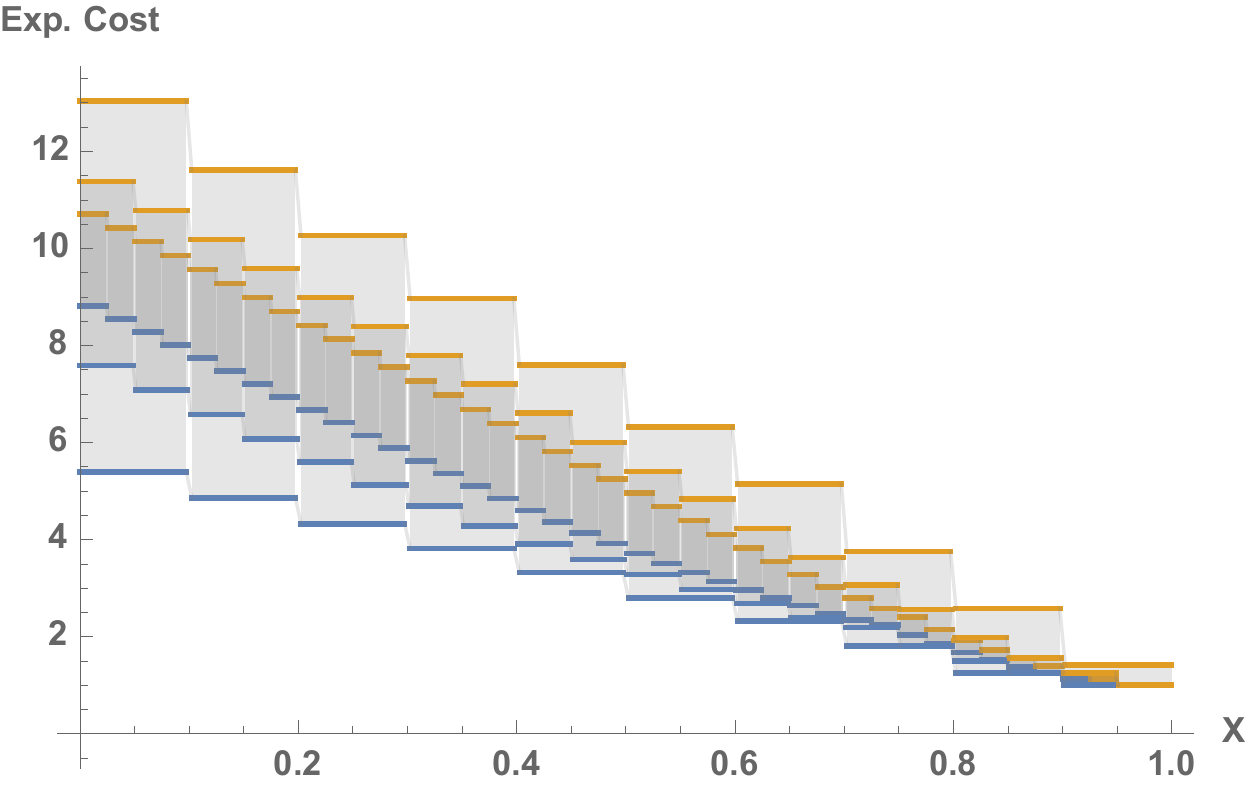}
   \includegraphics[scale=0.47]{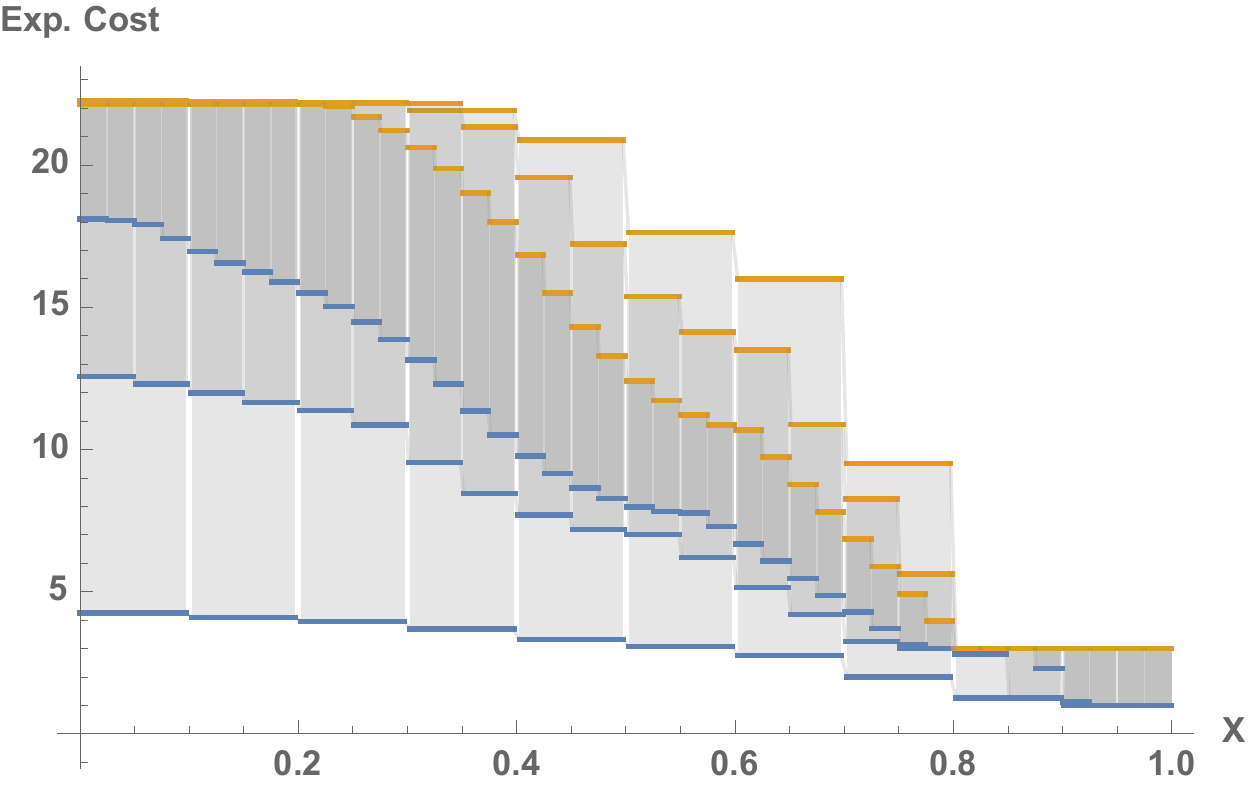}
  \caption{Upper (yellow) and lower (blue) expected cost functions of 
    IMDPs $\mdp_\partition(\Delta)$ for $\Delta \in \{ 0.1, 0.05, 0.025 \}$ on $\states'_{0}$ (left) and
    $\states'_{0.7}$ (right).}
  \label{fig:upperlower}
\end{figure}

Figure~\ref{fig:upperlower} shows the upper and lower expected costs that we obtain
from the induced IMDPs. One can see how the intervals narrow with successive partition refinements.
The bounds on the section $\states'_{0}$ are closer and converge more uniformly than on 
$\states'_{0.7}$. This shows that in the upper left region of the state space ($x<0.5, t\geq 0.7$)
the adversary has a greater influence on the process than at the lower part of the
state space ($x\sim 0$), and the difference between a cooperative and a non-cooperative
adversary is more pronounced.

\begin{figure}
  \centering
  \includegraphics[scale=0.3]{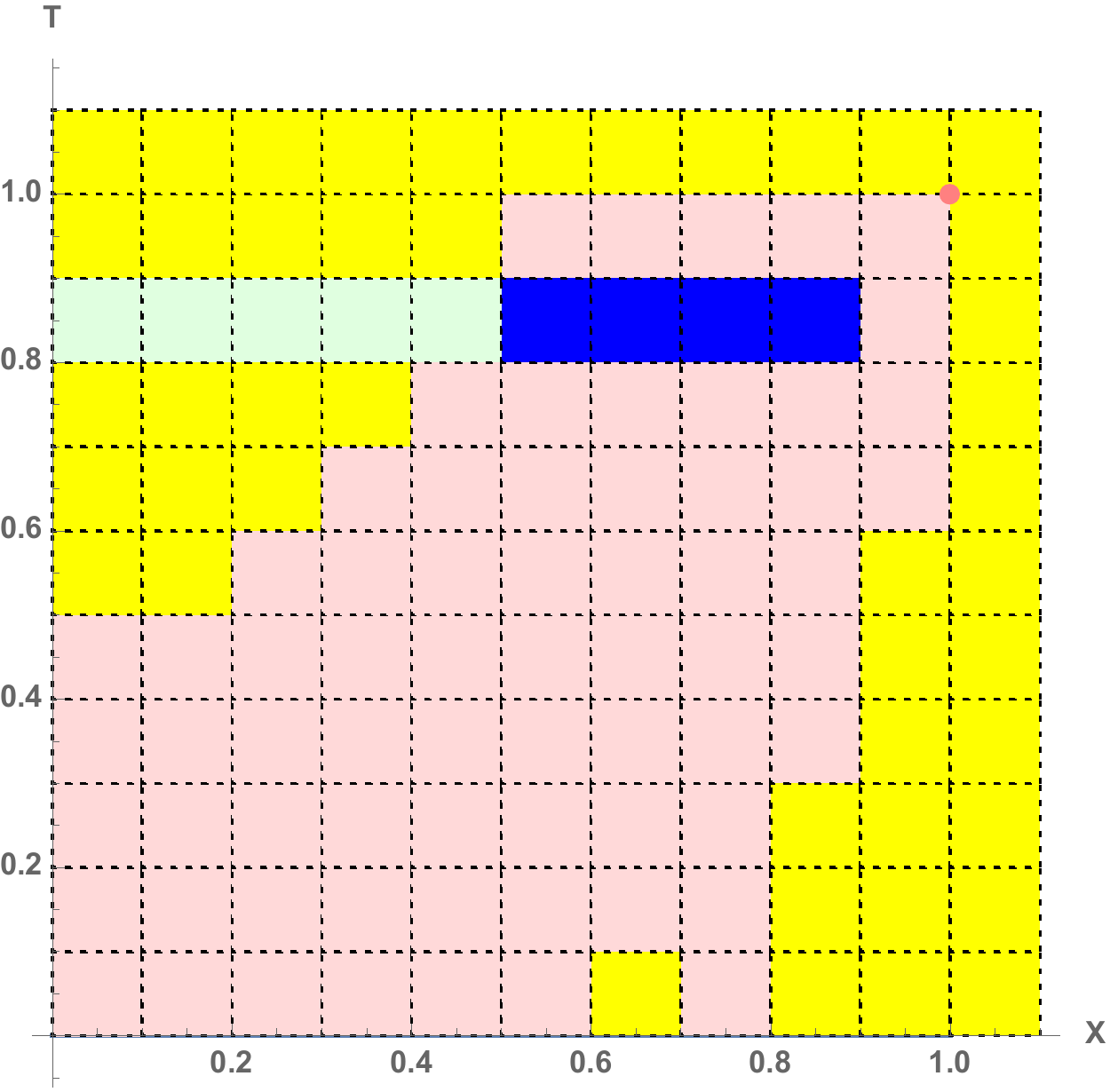}
  \includegraphics[scale=0.3]{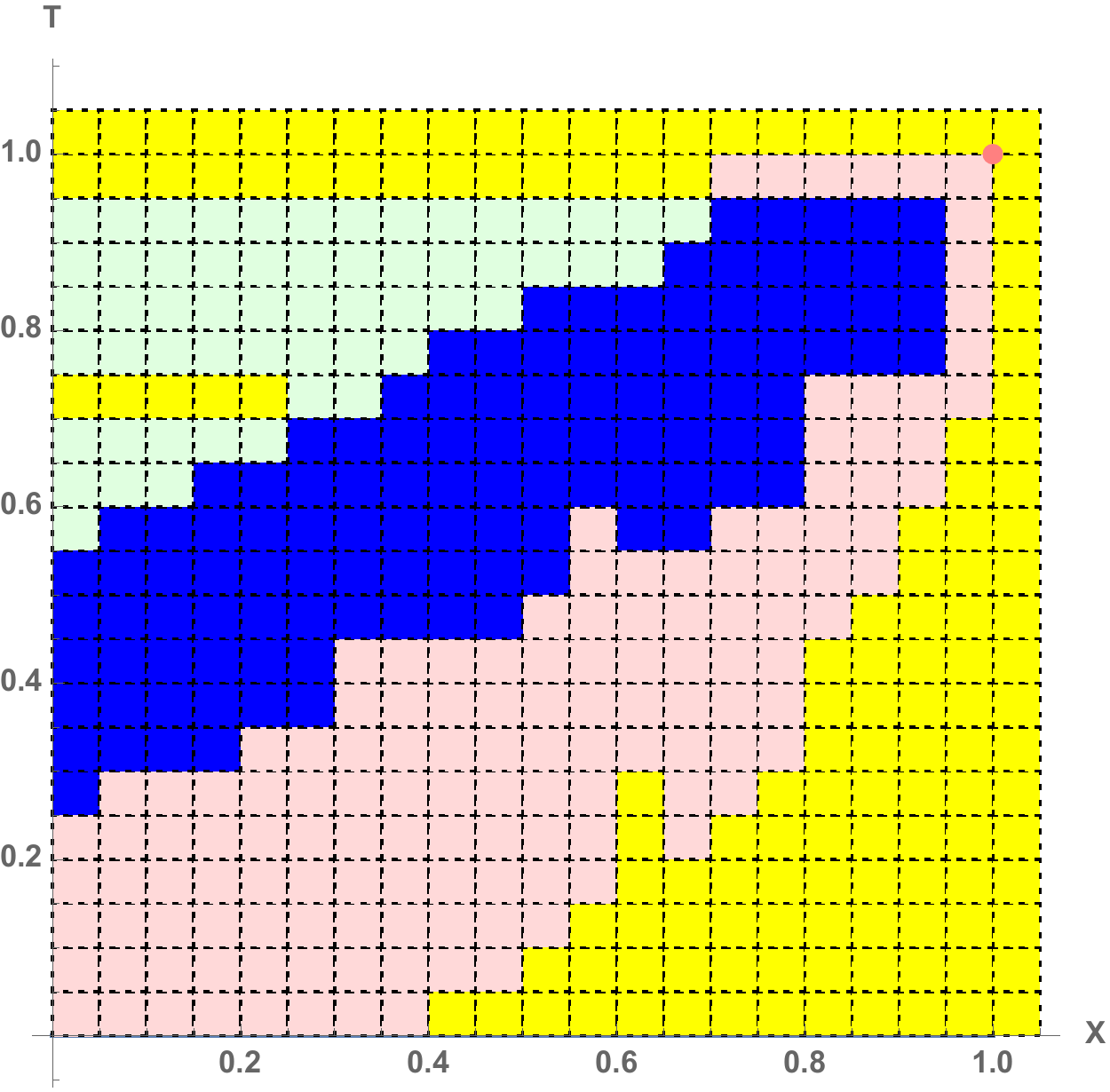}
  \includegraphics[scale=0.3]{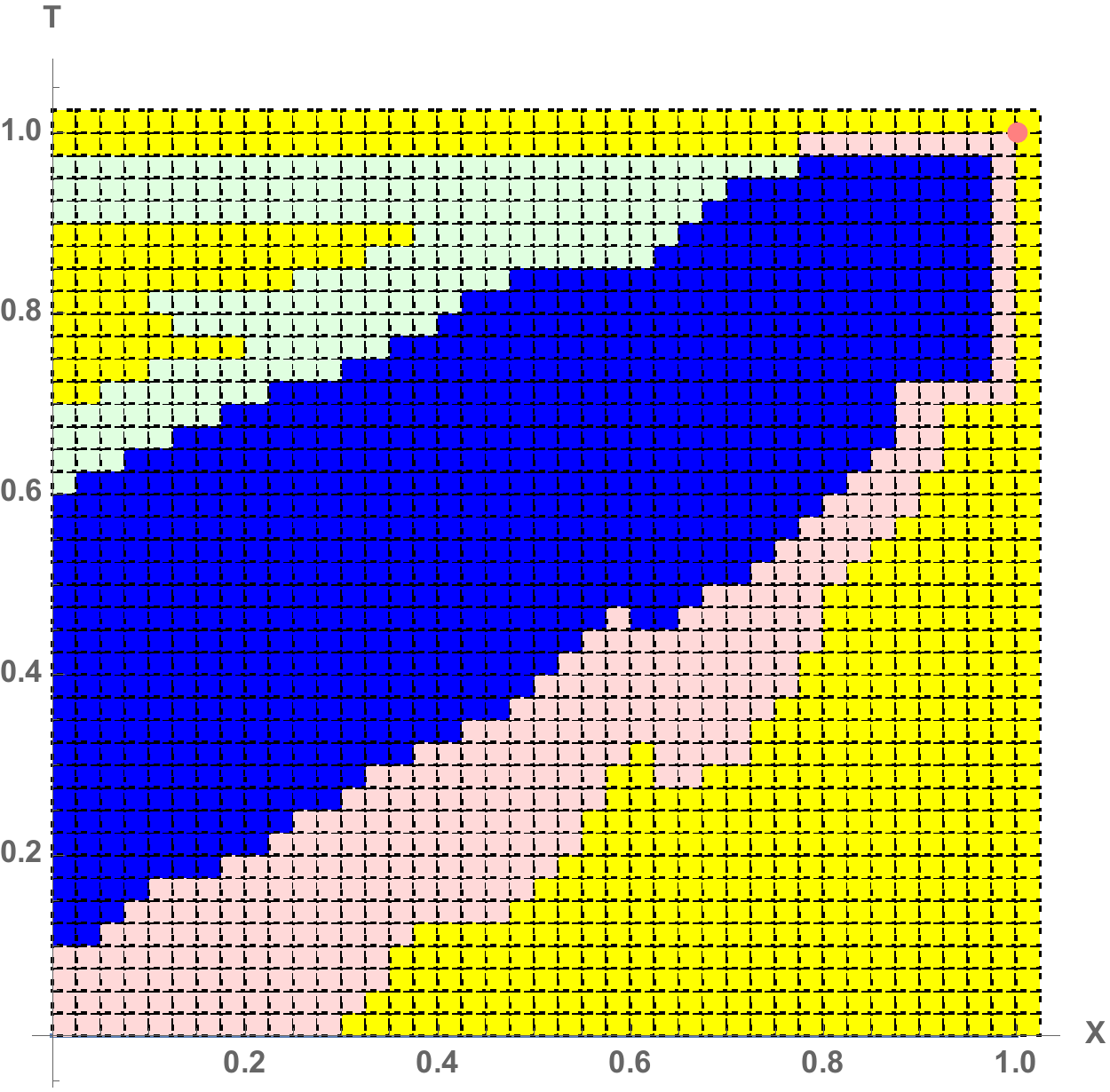} 
  \caption{Strategies obtained from lower and upper expected cost approximations for  $\mdp_\partition(\Delta)$
    for $\Delta = 0.1, 0.05, 0.025$ (left to right). }
  \label{fig:imdppolicy}
\end{figure}

Ultimately, induced strategies are of greater interest than the concrete cost functions.
Once upper and lower expectations define the same strategy, further refinement may
not be necessary.
Figure~\ref{fig:imdppolicy} illustrates  for the whole state space $\states$
the strategies $\sigma$  obtained from the lower (Equation (\ref{eq:Emin}))
and upper (Equation (\ref{eq:Emax})) approximations.
On regions colored blue and yellow, both strategies agree to take the \emph{fast} and \emph{slow}
actions, respectively. The regions colored light green are those where the lower bound
strategy chooses the \emph{fast} action, and the upper bound strategy the \emph{slow} action. Conversely
for the regions colored light red. One can observe how the blue and yellow areas increase in
size with successive partition refinements. However, this growth is not entirely monotonic:
for example, some regions in the upper left
that for $\Delta=0.1$ are yellow are sub-divided in successive refinements
$\Delta=0.05,0.025$ into regions that are partly yellow, partly light green. 

\subsection{Analysis of learned strategies}
%

\begin{figure}[tb]
  \centering
  \includegraphics[scale=0.47]{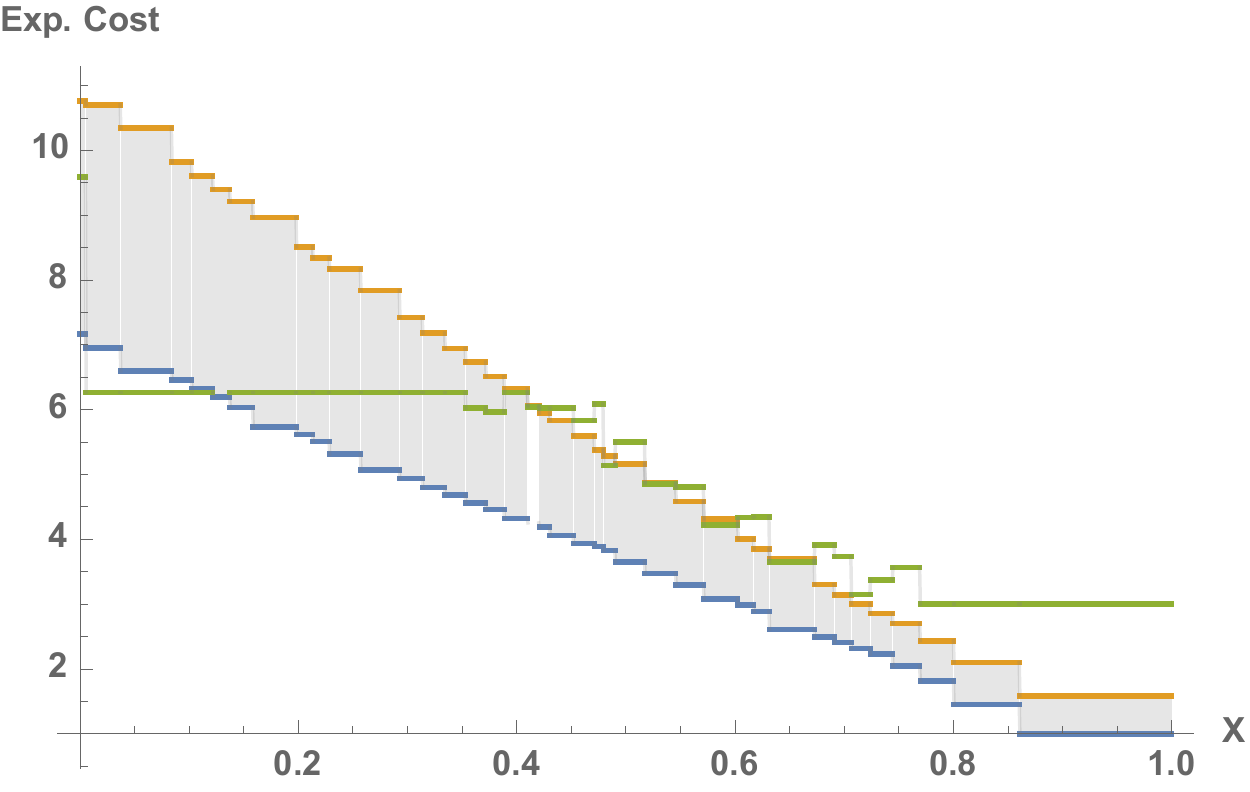}
  \includegraphics[scale=0.47]{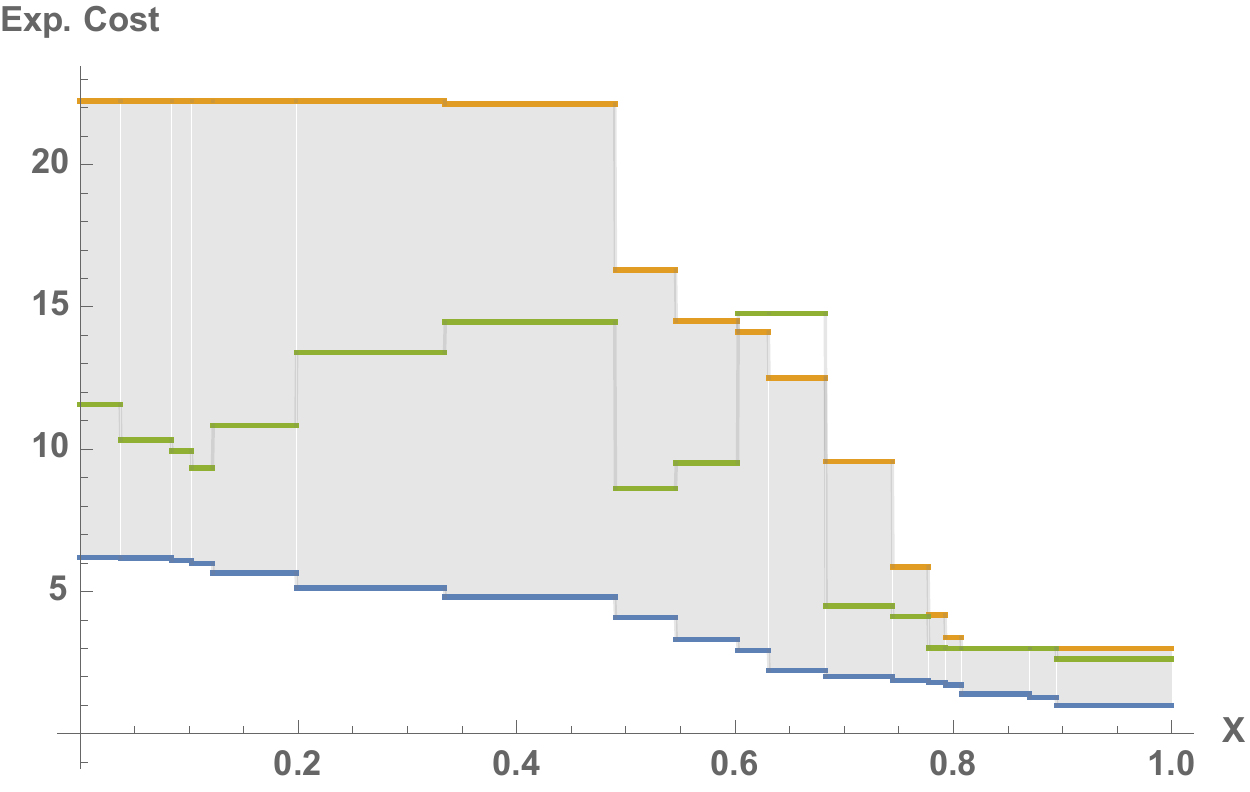}
  \caption{Expected cost functions along $S'_0$ (left) and $S'_{0.7}$ (right). Green: learned cost function;
   yellow/blue: upper/lower expected cost function obtained from IMDP.}\label{fig:upperlowerlearn}
\end{figure}

We now turn to partitions computed by the reinforcement learning method
developed in~\cite{jaeger2019teaching}, and a comparison of the learned cost functions and strategies with
those obtained from the induced IMDPs. 
We have implemented the semi-random walker in UPPAAL STRATEGO and used  reinforcement learning to
learn partitions, cost functions and strategies.
Our learning framework produces a sequence of refinements, based on sampling $100$ additional
runs for each refinement. In the following we consider the models learned after
$k=27$ and $k=205$ refinements. 

Figure~\ref{fig:upperlowerlearn} illustrates expected costs functions for the partition learned at
$k=205$. One can observe a strong correlation between the bounds and the learned costs. Nevertheless,
the learned cost function sometimes lies outside the given bounds. This is to  be expected, since
the random sampling process may produce data that is not sufficiently representative to estimate
costs for some regions. 

\begin{figure}
  \centering
  \includegraphics[scale=0.47]{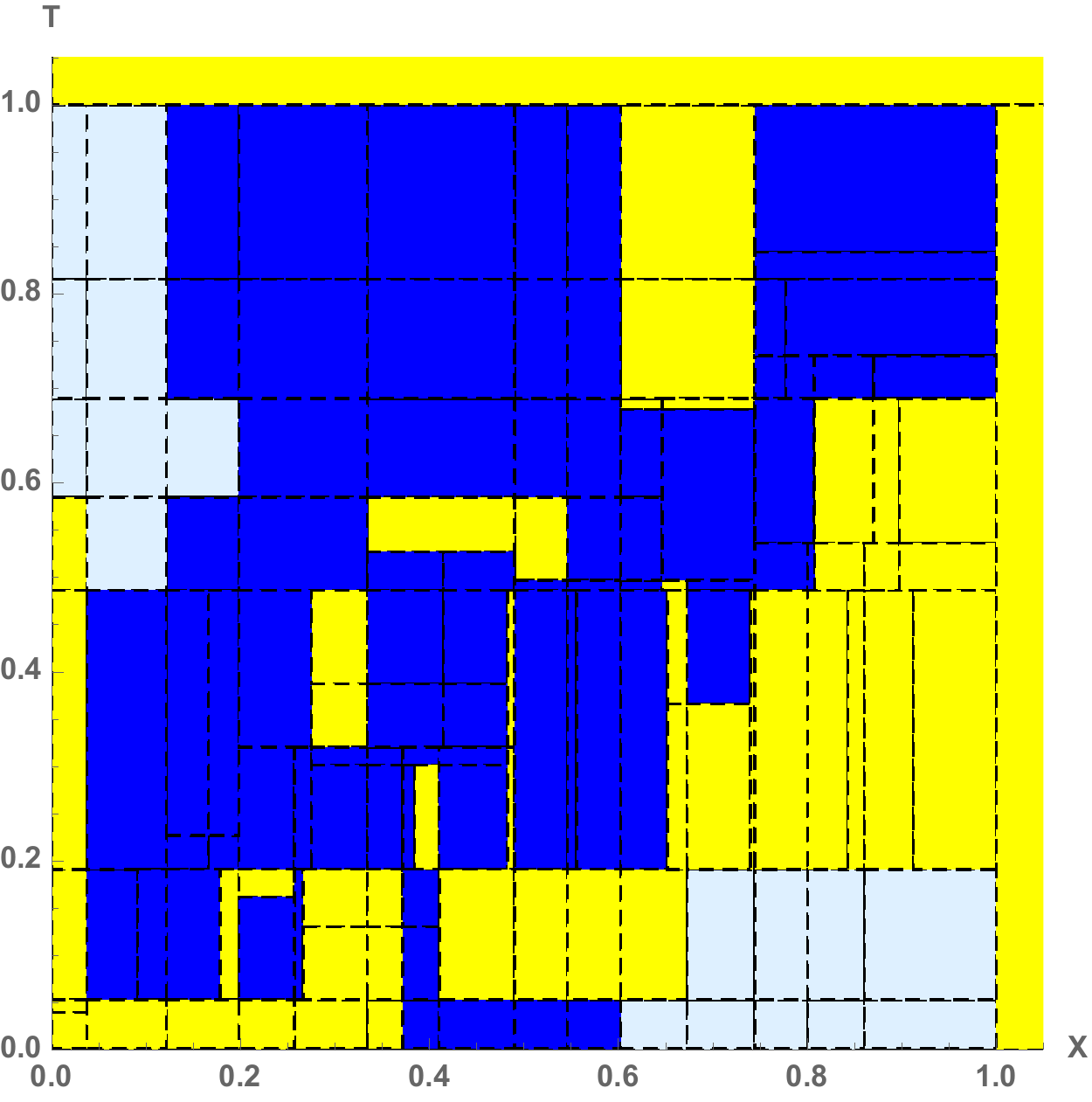}
\includegraphics[scale=0.47]{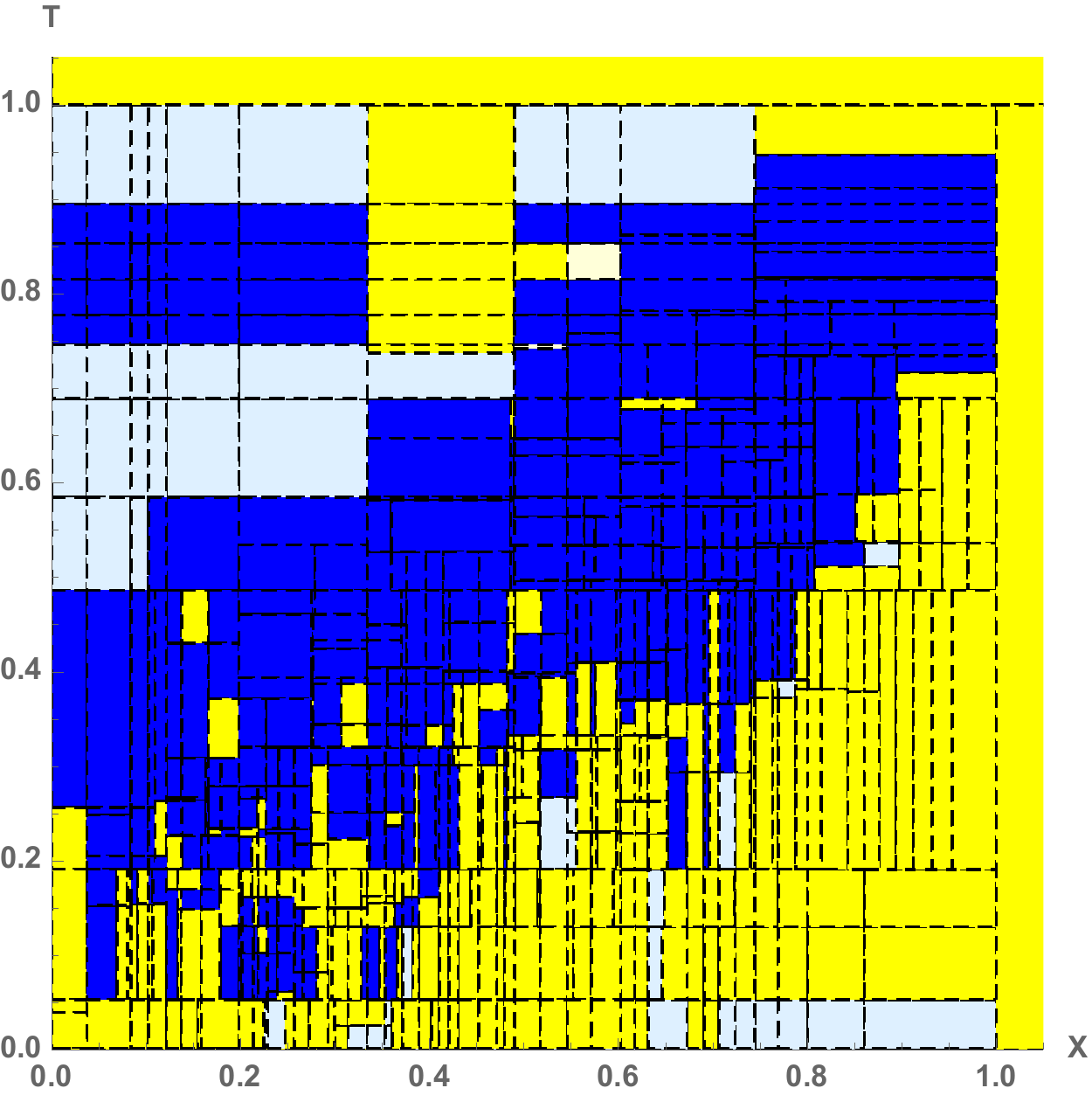}
  \caption{Comparison of the strategies obtained for the IMDP induced by the partition $\partition(27)$ (left) and $\partition(205)$ (right).} \label{fig:learnpolicy}
\end{figure}

Turning again to the strategies obtained on the whole state space, we first note that the
learned strategy at $k=205$, which is shown in Figure~\ref{randomwalkfig} (right)
exhibits an overall similarity with the strategies illustrated in Figure~\ref{fig:imdppolicy}, with
the \emph{fast} action preferred along a diagonal region in the middle of the state space.
To understand the differences between the learning and IMDP results, it is important to
note that in the learning setting $s_0=(0,0)$ is taken to be the initial state of interest, and
all sampling starts there. As a result, regions that are unlikely to be reached (under any choice
of actions) from this initial state will obtain very little relevant data, and therefore
unreliable cost estimates. This is not necessarily a disadvantage, if we want to learn
an optimal control strategy for processes starting at $s_0$. The value iteration process does not
take into account the distinguished nature of $s_0$.

Figure~\ref{fig:learnpolicy} provides a detailed picture of the consistency of the strategies learned at
$k=27$ and $k=205$ with the strategies obtained from value iteration over the same partitions.
Drawn in blue/yellow are those regions where the learned strategy picks the \emph{fast/slow} action, and
at least one of  upper or lower bound strategies  selects the same action. Light blue are those regions
where the learned strategy chooses the \emph{fast} action, but both IMDP strategies select \emph{slow}.
In a single region in the $k=205$ partition (drawn in light yellow) the learned strategy chooses the \emph{slow},
while both IMDP strategies select \emph{fast}. As Figure~\ref{fig:learnpolicy} shows, the areas of greatest
discrepancies (light blue) are those in the top left and bottom right, which are unlikely to be reached
from initial state $(0,0)$.


\section{Conclusion}

In this paper we have developed theoretical foundations for the approximation of Euclidean MDPs by
finite state space imprecise MDPs. We have shown that bounds on the cost function computed on the
basis of the IMDP abstractions are correct, and that for bounded time horizons they converge to the
exact costs when the IMDP abstractions are refined. We conjecture that this convergence also holds
for the total cost of (potentially) infinite runs.

The results we here obtained provide theoretical underpinnings for the learning approach
developed in~\cite{jaeger2019teaching}. Upper and lower bounds computed from induced IMDPs can
be used to check the accuracy of learned value functions. As we have seen, data sparsity and
sampling variance can make the learned cost functions fall outside computed bounds.
One can also use value iteration on IMDP approximations directly as a tool for computing cost functions
and strategies, which then would come with stronger guarantees than what we obtain through learning.
However, compared to the learning approach, this has 
important limitations:  first, we will usually only obtain a partial strategy that is uniquely defined
only where upper and lower bounds lead to the same actions. Second, we will require a full model of the underlying EMDP,
from which IMDP abstractions then can be derived, and the optimization problem over adversaries that is part of the
value iteration process must be tractable. Reinforcement learning, on the other hand, can also be applied to black box
systems, and its computational complexity is essentially independent of the complexities of the
underlying dynamic system.

\bibliographystyle{abbrvnat}
\bibliography{imdp}

\newpage
\appendix

\section{Total Variation Distance}
The following lemma collects some basic facts about
 total variation  distance:

 \begin{lemma}
   \label{lem:dtvbasic}
   Let $\partition$ be a finite set,
   and  $P,P'$ be distributions on $\partition$ with $\dtv(P,P')\leq\epsilon$.
   \begin{description}
   \item[A] Let $f,f'$ functions on $\partition$ with values in $\reals_{\geq  0}$ and
     $|f(\nu)-f'(\nu)|\leq \epsilon$ for all $\nu$. Then
     \begin{equation}
     \label{eq:diffeval}
     |\expected[f]-\expected'[f']| \leq \epsilon\cdot (1+2\max_{\nu\in\partition}f(\nu)),
   \end{equation}
   where $\expected,\expected'$ denote expectation under $P$ and $P'$, respectively.
 \item[B] For each $\nu\in\partition$ let $Q_\nu,Q_\nu'$ be distributions on a space $\states$
   (discrete or continuous), such that $\dtv(Q_\nu,Q_\nu')\leq\epsilon$ for all $\nu$. Then
   \begin{equation}
     \label{eq:diffmixture}
     \dtv( \sum_\nu P(\nu)Q_\nu , \sum_\nu P'(\nu)Q_\nu'   )\leq 3\epsilon.
   \end{equation}
   \end{description}
   
 \end{lemma}

 \begin{proof}
   For {\bf A} we write
   \begin{displaymath}
     |\expected[f]-\expected'[f']| \leq 
     |\expected[f]-\expected'[f]| + |\expected'[f]-\expected'[f']| 
   \end{displaymath}
   With
   \begin{multline*}
     |\expected[f]-\expected'[f]|=|\sum_\nu P(\nu)f(\nu) -  \sum_\nu P'(\nu)f(\nu)| \leq \\
     \sum_\nu  f(\nu)| P(\nu)-  P'(\nu)|  \leq \max_{\nu\in\partition}f(\nu) 2 \dtv(P,P')
   \end{multline*}
   and
   \begin{displaymath}
      |\expected'[f]-\expected'[f']|\leq\epsilon
    \end{displaymath}
    then (\ref{eq:diffeval}) follows.
    The proof for {\bf B} is very similar:
    \begin{multline*}
      \dtv\left( \sum_\nu P(\nu)Q_\nu , \sum_\nu P'(\nu)Q_\nu'   \right)\leq \\
      \dtv\left( \sum_\nu P(\nu)Q_\nu , \sum_\nu P(\nu)Q_\nu'    \right) +  \dtv\left( \sum_\nu P(\nu)Q_\nu' , \sum_\nu P'(\nu)Q_\nu'    \right).
    \end{multline*}
    Using the definition of total variation as $\dtv(P,P')=sup_{S\subseteq\states}|P(S)-P'(S)|$ the first term
    on the right can be bounded by $\epsilon$, and the second by $2\epsilon$. 
 \end{proof}

 \section{Proofs}

 \measurable*
  \begin{proof}
    For each $i$, $\pi \mapsto \cost(\pi_i)$ is  $(\borel^K\otimes 2^{\act})^{\infty} - {\borel}_+$
    measurable according to the measurability condition on $\cost$. It follows that
    also $\cost^{(N)}(\pi):= \sum_{i=1}^N \cost(\pi_i)$ is measurable for every $N$.
    Since $\costinfty$ is the supremum of the $\cost^{(N)}$, it
    is measurable~\cite[Proposition 2.1.4]{Cohn1980}.
  \end{proof}

  \strategykernel*
  \begin{proof}
  For fixed $(s,a)$, $T_\sigma$ is a probability measure on $\borel^K\times 2^{\act}$ by
  construction. To show that for fixed $(B,A)$ the function $(s,a)\mapsto T_\sigma((s,a),(B,A))$
  is measurable, we only need to consider the case of singletons $A=\{a'\}$. By the measurability
  of $\sigma(\cdot)(a')$ we can express $\sigma(\cdot)(a')$ as the supremum of a monotone increasing
  sequence of simple
  measurable functions $\sigma^{(k)}(\cdot)(a')$\footnote{Recall that a simple function is a finite
    weighted sum of indicator functions of measurable sets}\cite[Theorem 13.5]{Billingsley1986}.
  For each  $\sigma^{(k)}$ the integral
  $\int_B \sigma^{(k)}(s')(\{a'\}) T(s,a,ds')$ then decomposes into a weighted sum of
  integrals of the form $\int_{B\cap C_i} T(s,a,ds') = T(s,a,B\cap C_i)$, which are measurable in $s$
  according to Definition~\ref{def:mdp}. Finally,
  \begin{displaymath}
    \int_B \sup_k\sigma^{(k)}(s')(\{a'\}) T(s,a,ds')  =
    \sup_k  \int_B\sigma^{(k)}(s')(\{a'\}) T(s,a,ds')
  \end{displaymath}
  by the monotone convergence theorem~\cite[Theorem 16.2]{Billingsley1986}, and measurability
  follows from the measurability of the supremum of measurable
  functions~\cite[Theorem 13.4]{Billingsley1986}.
\end{proof}

\LEmeasurable*
\begin{proof}
By \cite[Theorem 13.5]{Billingsley1986} we can express $E$ as the supremum of a 
monotone increasing sequence of simple measurable functions $E^{(k)}$. 
For each $E^{(k)}$ the integral $\int_{t \in \states} E^{(k)}(t) \,  T(s,a, \dee t)$ then 
decomposes into a weighted sum of integrals of the form 
$\int_{t \in C_i} T(s,a, \dee t) = T(s,a, C_i)$, for some measurable set $C_i \subseteq \states$, 
which are measurable according to Definition~\ref{def:mdp}. 
Since $E = \sup_k E^{(k)}$, by the monotone convergence theorem~\cite[Theorem 16.2]{Billingsley1986}, 
\begin{align*}
    \viop E(s) &=  \min_{a\in\act} \left( \cost(s,a) +
  \sup_k \int_{t \in \states}  E^{(k)}(t) \,  T(s,a, \dee t)  \right) \,,
  \\
  \viop^\sigma E(s) &= \sum_{a \in \act} \sigma(s)(a) \cdot \left( \cost(s,a) + 
  \sup_k \int_{t \in \states} E^{(k)}(t)  \,  T(s,a, \dee t)  \right) \,,
 \end{align*}
From the above, measurability of $\viop E$ follows from the measurability of 
$\cost(\cdot,a)$, for all $a\in\act$, and of minima of measurable functions~\cite[Theorem 13.4]{Billingsley1986}.
Measurability of $\viop^\sigma E$ follows similarly by additionally noticing that for any strategy 
$\sigma$, the $[0,1]$-valued function $\sigma(\cdot)(a)$ is measurable, for all $a \in \act$.
\end{proof}

\deteministicstrategies*
\begin{proof}
$\inf_\sigma \viop^\sigma \leq \viop$ follows by noticing that $\viop = \inf_d \viop^d$, where $d$ ranges only over
deterministic strategies.
To establish the reverse equality, notice that, for all $\sigma$ and $s \in \states$, 
$\sum_{a\in\act} \sigma(s)(a) = 1$.
\begin{align*}
\viop E(s) 
&= \sum_{a\in\act} \sigma(s)(a) \cdot \viop E(s) \\
&\leq  \sum_{a \in \act} \sigma(s)(a) \cdot \left( \cost(s,a) + \int_{t \in \states} E(t) \,  T(s,a, \dee t)  \right)
= \viop^\sigma(E)(s) \,.
\end{align*}
Thus, $\viop \leq \viop^\sigma$, for all strategies $\sigma$. From this we obtain 
$\inf_\sigma \viop^\sigma \geq \viop$.
\qed
\end{proof}

\volterraStrategy*
\begin{proof}
We have to show that the following holds for all states $s\in\states$:
\begin{equation}
  \label{eq:volterra}
  \expected_{\sigma}(\cost, s) = 
  \sum_{a \in \act} \sigma(s)(a) \cdot \left( \cost(s,a) + \int_{t \in \states} \expected_{\sigma}(\cost, t) \,  T(s,a, \dee t)  \right) \,.
\end{equation}
By monotone convergence theorem and linearity of the integral, we have
\begin{align}
\expected_{\sigma}(\cost, s) 
&= \sup_N \int_{\pi \in \runs} \sum_{i=1}^N \cost(\pi_i) \, \pruns{s,\sigma}(\dee \pi) \notag \\
&= \int_{\pi \in \runs} \cost(\pi_1) \pruns{s,\sigma}(\dee \pi) + \sup_N \int_{\pi \in \runs} \sum_{i=2}^N \cost(\pi_i) \, \pruns{s,\sigma}(\dee \pi) \notag \\
&= \int_{\pi \in \runs} \cost(\pi_1) \pruns{s,\sigma}(\dee \pi) + 
\int_{\pi \in \runs} \costinfty(\run_{>1}) \, \pruns{s,\sigma}(\dee \run) \,.
\label{eq:volterra-split}
\end{align}

By definition, the first expectation in \eqref{eq:volterra-split} is just 
\begin{align*}
\int_{\pi \in \runs} \cost(\pi_1) \, \pruns{s,\sigma}(\dee \pi) 
&= \sum_{a\in\act} \cost(s,a) \cdot \sigma(s)(a) \,.
\end{align*}
and by a change of variable in the integral, the second expectation in \eqref{eq:volterra-split} is
\begin{align*}
\int_{\pi \in \runs} \costinfty(\run_{>1}) \, \pruns{s,\sigma}(\dee \run)  
&= \sum_{a\in\act} \sigma(s)(a) \cdot \int_{t \in \states}
\left( \int_{\run \in \runs} \costinfty(\run) \pruns{t, \sigma}(\dee \run) \right) \, T(s,a, \dee t) \\
&= \sum_{a\in\act} \sigma(s)(a) \cdot 
 \int_{t \in \states} \expected_\sigma(\cost,t) \, T(s,a, \dee t) \,.
\end{align*}
Thus, \eqref{eq:volterra} follows.
\qed
\end{proof}

\leastprefixpoint*
\begin{proof}
By Lemma~\ref{lem:deteministicstrategies}, Proposition~\ref{prop:volterraStrategy} and monotonicity 
of $\viop^\sigma$ we have 
\begin{equation*}
\expected(\cost, \cdot) =
\inf_\sigma \expected_\sigma(\cost, \cdot) = 
\inf_\sigma \viop^\sigma \expected_\sigma(\cost, \cdot) \geq
\inf_\sigma \viop^\sigma \expected(\cost, \cdot) = 
\viop \expected(\cost, \cdot) \,.
\end{equation*}

Next we prove that if $E \geq \viop E$, then $E \geq \expected(\cost, \cdot)$.
By induction on $n \geq 1$, we prove that, for all $s \in \states$ and strategies $\sigma$
\begin{equation}
(\viop^\sigma)^n E(s) \geq 
\int_{\run \in \runs} \sum_{i = 1}^n \cost(\run_i) \,  \pruns{s,\sigma}(\dee \pi) \,.
\label{eq:iter}
\end{equation}
The base case $n = 1$ follows by definition of $\pruns{s,\sigma}(\dee \pi)$ and because $E$ is positive:
\begin{align*}
\int_{\pi \in \runs} \cost(\pi_1) \, \pruns{s,\sigma}(\dee \pi) 
&= \sum_{a\in\act} \cost(s,a) \cdot \sigma(s)(a) \\
&\leq \sum_{a \in \act} \sigma(s)(a) \cdot \left( \cost(s,a) + \int_{t \in \states} E(t) \,  T(s,a, \dee t)  \right) \,.
\end{align*}
As for the inductive step, assume \eqref{eq:iter} holds for $n \geq 1$. Then
\begin{align*}
&\int_{\run \in \runs} \sum_{i = 1}^{n+1} \cost(\run_i) \,  \pruns{s,\sigma}(\dee \pi) \\
&= \int_{\run \in \runs} \cost(\run_1) \,  \pruns{s,\sigma}(\dee \pi) +
      \int_{\run \in \runs} \sum_{i = 2}^{n+1} \cost(\run_i) \,  \pruns{s,\sigma}(\dee \pi) \\
&= \sum_{a \in \act} \sigma(s)(a) \cdot \left(  
      \cost(s,a) + \int_{t \in \states} 
          \left( \int_{\run \in \runs} \sum_{i = 1}^{n} \cost(\run_i) \,  \pruns{t,\sigma}(\dee \pi) \right)
          \, T(s,a, \dee t)
      \right) \\
&\leq \sum_{a \in \act} \sigma(s)(a) \cdot \left(  
      \cost(s,a) + \int_{t \in \states} (\viop^\sigma)^n E(t) \, T(s,a, \dee t) \right) \\
&= (\viop^\sigma)^{n+1} E(s) \,.
\end{align*}

Let $d$ be the deterministic strategy such that 
$\viop E= \viop^d E$. By hypothesis, $E \geq \viop E$, and by monotonicity of $\viop^d$, we obtain 
$E \geq (\viop^d)^n E$, for all $n \geq 1$.
Thus, by \eqref{eq:iter} and monotone convergence theorem, for all $s \in \states$
\begin{align*}
E(s) 
\geq \sup_{n \geq 1} (\viop^d)^n E(s) 
\geq  \sup_{n \geq 1} \int_{\run \in \runs} \sum_{i = 1}^n \cost(\run_i) \,  \pruns{s,d}(\dee \pi)
= \expected_d(\cost, s) \,.
\end{align*}
Since $\expected(\cost, s) = \inf_{\sigma} \expected_\sigma(\cost, s)$, from the above
we have $E(s) \geq \expected(\cost, s)$.
\qed
\end{proof}

\policyiter*
 \begin{proof}
   The chain $ \bot \leq L^1  \leq L^2 \leq \dots$ is monotonically increasing.
   This is immediate from $\bot \leq \viop \bot$ and monotonicity of the operator $\viop$.
   
   Next we show that $L$ is a fixed point of the $\viop$ operator.
   Clearly, $\bot \leq \viop L$, and by monotonicity of $\viop$, for all $n \geq 1$, $L^n \leq \viop L$.
   Hence $L \leq \viop L$. Now we establish $\viop L \leq L$. If $L(s) = \infty$, the inequality holds trivially
   on $s$. Assume $L(s) < \infty$. Then there exist a sequence $(a_n)_{n \geq 0} \in \act$ such that
   \begin{align*}
   L(s) 
   &= \sup_{n \geq 0} \min_{a\in\act} \left( \cost(s,a) +
   \int_{t \in \states} L^n(t) \,  T(s,a, \dee t)  \right) \\
   &= \sup_{n \geq 0} 
   \left( \cost(s,a_n) + \int_{t \in \states} L^n(t) \,  T(s,a_n, \dee t) \right) \,.
   \end{align*}
   
   Let $\states^\infty = \{ t \in \states \mid L(t) = \infty \}$. In the following we show that 
   \begin{equation}
     \exists N \geq 0 \text{ such that, } \forall n \geq N .\, T(s, a_n, \states^\infty) = 0 \,.
     \label{eq:goodactionselection}
   \end{equation}
%
   
   If $S^\infty = \emptyset$, \eqref{eq:goodactionselection} holds trivially. Let 
   $S^\infty \neq \emptyset$. Assume by contradiction that for all $N \geq 0$ there exists $n \geq N$
   such that $T(s, a_n, \states^\infty) > 0$. This is equivalent to the existence of a subsequence 
   $(a_k)$ such that for all $a_k$, $T(s, a_k, \states^\infty) > 0$.
   For $b \in \reals$ and $n \geq 0$, 
   denote by $E_b^n$ the set $\{ t \in \states \mid L^n (t) \geq b \}$. Then,  
   \begin{equation*}
   \states^\infty 
   = \{ t \in \states \mid \forall b \in \reals.\, \exists n \geq 0. \, L^n (t) \geq b \} 
   = \bigcap_{b \in \reals} \bigcup_{n \geq 0} E_b^n \,.
   \end{equation*}
   Moreover, for all $b,b' \in \reals$ and $n \geq 0$, if $b \geq b'$ then $E_{b}^n \subseteq E_{b'}^n$
   and by monotonicity of the operator $\viop$, $E_b^{n} \subseteq E_b^{n+1}$. Thus, by \cite[Theorem~10.2]{Billingsley1986}, for all $a_k$ and $b \in \reals$
   \begin{gather}
     T(s, a_k, \states^\infty) = \inf_{b \in \reals} T(s, a_k, \bigcup_{n \geq 0} E_b^n) \,, 
     \label{eq:inf}
     \\
     T(s, a_k, \bigcup_{n \geq 0} E_b^n) = \sup_{n \geq 0} T(s, a_k, E_b^n) \,.
     \label{eq:sup}
   \end{gather}
   Since $T(s, a_k, \states^\infty) > 0$, by \eqref{eq:inf}, , for all $a_k$, $T(s, a_k, \bigcup_{n \geq 0} E_b^n) > 0$. Consequently, by \eqref{eq:sup}, for all $b \in \reals$, exist $k'$ such that $T(s, a_{k'}, E_b^{k'}) \geq 0$. Thus, by
   \begin{equation*}
   L(s) \geq \int_{t \in E_b^{k'}} L^{k'}(t) \,  T(s,a_{k'}, \dee t) 
   \geq \int_{t \in E_b^{k'}} b \,  T(s,a_{k'}, \dee t) = b \cdot T(s,a_{k'}, E_b^{k'})
   \end{equation*}
   and the fact that $b$ can assume arbitrarily large values, $L(s) = \infty$.
   This contradicts our initial assumption that $L(s) < \infty$. 
   Therefore \eqref{eq:goodactionselection} must hold.
   
   By \eqref{eq:goodactionselection}, for all
   $n \geq N$, $\int_{t \in \states} L^n(t) \,  T(s,a_n, \dee t) = 
   \int_{t \in \states} (L(t) - L(t)) L^n(t) \,  T(s,a_n, \dee t)$.
   Thus the following hold:
   \begin{align*}
   L(s) 
   &= \sup_{n \geq N} 
   \left( \cost(s,a_n) + \int_{t \in \states} L^n(t) \,  T(s,a_n, \dee t) \right) \\
   &= \sup_{n \geq N} 
   \left( \cost(s,a_n) + \int_{t \in \states} L(t) \,  T(s,a_n, \dee t) \right) + \Delta(s) \\
   &\geq \sup_{n \geq N} \min_{a \in \act}
   \left( \cost(s,a) + \int_{t \in \states} L(t) \,  T(s,a, \dee t) \right) + \Delta(s) \\
   &= \viop L(s) + \Delta(s) \,,
   \end{align*}
   where
   \begin{equation*}
      \Delta(s) = \sup_{n \geq N} \int_{t \in \states} \Big( L^n(t) - L(t) \Big) \,  T(s,a_n, \dee t) \,.
   \end{equation*}
   Hence, if $\Delta(s) = 0$, we get $L(s) \geq \viop L(s)$.
   
   The finiteness of $\act$ ensures the existence of an action $a' \in \act$ repeating 
   infinitely often in $(a_n)_{n \geq N}$. Thus exists a subsequence $(n_k)$ such
   that, for all $n_k$
   \begin{equation*}
   \int_{t \in \states} \Big( L^{n_k}(t) - L(t) \Big) \,  T(s,a_{n_k}, \dee t) 
   = 
   \int_{t \in \states} \Big( L^{n_k}(t) - L(t) \Big) \,  T(s,a', \dee t) \,.
   \end{equation*}
   Therefore
   \begin{equation*}
      \Delta(s) = \sup_{n \in (n_k)} \int_{t \in \states} \Big( L^n(t) - L(t) \Big) \,  T(s,a', \dee t) \,,
   \end{equation*}
   and, by monotone convergence theorem, $\Delta(s) = 0$.
   
   \medskip
   Finally, we show that $L = \expected(\cost, \cdot)$.
   By monotonicity of $\viop$, for all $n \geq 0$, we have $L^n \leq \expected(\cost, \cdot)$.
   Hence $L \leq  \expected(\cost, \cdot)$.
   The reverse inequality $L  \geq \expected(\cost, \cdot)$ follows by Proposition~\ref{prop:leastprefixpoint}
   since $L \geq \viop L$.
   \qed
 \end{proof}

\viimdp*
 \begin{proof}
   
   {\bf Step 1}: The sequence $L^{\minmax,k}$ is monotonically increasing: this is immediate from the
   facts that $L^{\minmax,1}\geq \bot $ because $\cost^{\minmax}\geq 0$, and
   $C\geq C' \Rightarrow  \viop^{\minmax} C \geq \viop^{\minmax} C'$.

   {\bf Step 2}: We show that $ L^{\minmax}$ is a fixed point of the $\viop$ operator.
   Let
   \begin{displaymath}
     \states^{\minmax,\infty}:=\{s\in\states: L^{\minmax}(s)=\infty\}\ \ \mbox{and}\ \
     \states^{\minmax,<\infty}:=\states\setminus\states^{\minmax,\infty}.
   \end{displaymath}

   By monotonicity, for $s\in \states^{\minmax,\infty}$ we have $\viop^{\minmax} L^{\minmax}(s)=\infty$.
   Now let  $s\in \states^{\minmax,<\infty}$. We  define separately for the two
   cases of $\minmax$:
     \begin{displaymath}
     \begin{array}{lll}
       \act^{\min,<\infty}(s) & :=  & \{a\in\act| \exists T\in T^*(s,a)\ \mbox{s.t.}\ T(\states^{\min,\infty})=0\}, \\
       \act^{\max,<\infty}(s) & :=  & \{a\in\act| \forall T\in T^*(s,a):\  T(\states^{\max,\infty})=0\}. 
     \end{array}
   \end{displaymath}
   The set $\act^{\minmax,<\infty}(s)$  is non-empty (the closedness of $T^*$ is again required here), and 
   we can limit the optimization to actions that after the adversary's choice
   do not lead to infinite cost states:
   \begin{multline}
     \viop^{\minmax} L^{\minmax}(s)= \\
     \min_{a\in\act^{\minmax,<\infty}(s)}\left( \cost^{\minmax}(s,a)+
       \minmax_{T\in T^*(s,a)}\sum_{s'\in \states^{\minmax,<\infty}}T(s') L^{\minmax}(s')  \right).
     \label{eq:local350}
   \end{multline}
   Moreover, the restriction of the minimization
   to actions from $\act^{\minmax,<\infty}$ also already is valid for the definition of
   $\viop^{\minmax} L^{\minmax,n}(s)$ for all sufficiently large $n$. 
 We have that $L^{\minmax,n}\rightarrow L^{\minmax}$ uniformly on the (finite) set
 $\states^{\minmax,<\infty}$. It follows that for all $s\in \states^{\minmax,<\infty}$ and
 $a\in\act^{\minmax,<\infty}$:
 \begin{multline}
   \label{eq:local352}
   \minmax_{T\in T^*(s,a)}\sum_{s'\in \states^{\minmax,<\infty}}T(s') L^{\minmax}(s') = \\
   \lim_{n\rightarrow\infty}\minmax_{T\in T^*(s,a)}\sum_{s'\in \states^{\minmax,<\infty}}T(s') L^{\minmax,n}(s').
 \end{multline}
   
With (\ref{eq:local350}) it then follows that
 \begin{displaymath}
   \viop^{\minmax} L^{\minmax}(s)=
   \lim_{n\rightarrow\infty}  \viop^{\minmax} L^{\minmax,n}(s) = L^{\minmax}(s).
 \end{displaymath}

 {\bf Step 3}: $\expected^{\minmax}$  is the least fixed-point of the $\viop^{\minmax}$-operator. That
 $\expected^{\minmax}$ is a fixed-point follows
 immediately from our restriction to stationary and memoryless strategies.

 Let $C\geq 0$ be an arbitrary fixed-point of  $\viop^{\minmax}$. Recalling (\ref{eq:sigmaC}), we
 can then write 
 \begin{equation}
   \label{eq:local400}
   C(s)=
   \minmax_{\alpha_T}\left(\cost^{\minmax}(s,\sigma^{\minmax}(C)(s))+\sum_{s'}\alpha_T(s,\sigma^{\minmax}(C))(s')C(s')\right).
 \end{equation}
 Un-rolling this recurrence for $n$ steps gives
  \begin{equation}
   \label{eq:local410}
   C(s)=
   \minmax_{\alpha_T}\left(
     \sum_{k=1}^n \expected^{(k)}_{s,\sigma^{\minmax}(C),\alpha_T}(\cost^{\minmax}(s_k,\sigma^{\minmax}(C)(s_k)))
   +\expected^{(n+1)}_{s,\sigma^{\minmax}(C),\alpha_T}(C)
     \right),
 \end{equation}
 where  $\expected^{(k)}_{s,\sigma^{\minmax}(C),\alpha^{\minmax}_T(C)}$ is the expectation over the state distribution
 at step $k$ defined by $\sigma^{\minmax}(C)$, $\alpha^{\minmax}_T(C)$ and initial state $s$. Since $C$ is
 non-negative, we can lower-bound (\ref{eq:local410}) by dropping the expectation $\expected^{(n+1)}_{s,\sigma^{\minmax}(C),\alpha_T}(C)$. Then taking the limit $n\rightarrow \infty$ gives
 \begin{multline*}
   C(s)\geq
   \minmax_{\alpha_T}\lim_{n\rightarrow\infty}\left(
     \sum_{k=1}^n \expected^{(k)}_{s,\sigma^{\minmax}(C),\alpha_T}(\cost^{\minmax}(s_k,\sigma^{\minmax}(C)(s_k)))
   \right) \\
   \geq
   \min_{\sigma}\minmax_{\alpha_T}\lim_{n\rightarrow\infty}\left(
     \sum_{k=1}^n \expected^{(k)}_{s,\sigma,\alpha_T}(\cost^{\minmax}(s_k,\sigma(s_k)))
   \right)\\
   =\expected^{\minmax}(\cost^{\minmax}(\pi),s),
 \end{multline*}
 where $\sigma^{\minmax}$ in the second line now is the strategy that actually achieves the minimum in 
 (\ref{eq:Emin}), respectively (\ref{eq:Emax}).
 
 {\bf Step 4}: Putting things together:  From the monotony of
 the $\viop^{\minmax}$ operator and the fixed point property of $ \expected^{\minmax}$ it  follows that
 $\expected^{\minmax}\geq L^{\minmax}$. From the fixed point property of
 $L^{\minmax}$ and the minimality property of $\expected^{\minmax}$ it follows that
 $\expected^{\minmax}\leq L^{\minmax}$.
 \end{proof}

\correctbounds*

 \begin{proof}
   We first consider  the inequality 
$\expected_{\mdp_\partition}^{\min}(\cost_\partition^*,[s]_\partition)\leq \expected_\mdp(\cost,s)$. 
Let $L^{\min,k}_{\mdp_\partition}$ denote the cost function obtained after
the $k$'th round of value iteration in $\mdp_\partition$, and $L^k_{\mdp}$ denote the cost function
after the  $k$'th round of value iteration in $\mdp$.
     By induction we show that $ L^{\min,k}_{\mdp_\partition}(\nu)\leq   L^k_{\mdp}(s)$ for all
     $k\geq 0$ and all $\nu,s$ with $s\in\nu$. For $k=0$ this is true according to the initialization.
     For $k+1$ we can then write:
     \begin{multline}
       L^{\min,k+1}_{\mdp_\partition}(\nu)  =
       \min_{a\in\act} \min_{s\in\nu}\left(\cost^{\min}(\nu,a)+
         \sum_{\nu'}T_\partition(s,a)(\nu') L^{\min,k}_{\mdp_\partition}(\nu')  \right)\\
       =  \min_{s\in\nu} \min_{a\in\act}\left(\cost^{\min}(\nu,a)+
         \sum_{\nu'}\int_{\nu'}T(s,a)(t)L^{\min,k}_{\mdp_\partition}(\nu')dt  \right)\\
       \leq \min_{s\in\nu} \min_{a\in\act}\left(\cost(s,a)+
         \sum_{\nu'}\int_{\nu'}T(s,a)(t)  L^k_{\mdp}(t)dt  \right)
       = \min_{s\in\nu} L^{k+1}_{\mdp}(s)
     \end{multline}
     The inequality $ \expected_\mdp(\cost,s) \leq
     \expected_{\mdp_\partition}^{\max}(\cost_\partition^*,[s]_\partition)$ is proven in the same way, here
     noting that $\min_a\max_s \geq \max_s\min_a$. The same line of argument can also be used to establish
     the inequalities (\ref{eq:AB1}),(\ref{eq:AB2}).
\qed
 \end{proof}

 \mainlem*
 \begin{proof}
   By induction on $N$. For $N=1$ let $\delta$ be such that for all
   $\partition$ with $\delta(\partition)\leq\delta$ it holds that  for all $\nu\in\partition$:
   $\cost^{\max}(\nu,a)-\cost^{\min}(\nu,a)\leq \epsilon$, and
   $\dtv(T(s,a),T(s',a))\leq \epsilon$ for all $s,s'\in\nu$. Then (\ref{eq:ENeq}) holds because
   \begin{displaymath}
      |\ENtau{N}{+}- \ENtau{N}{-}| \leq \cost^{\max}(\nu_1,\tau(\nu_1))-\cost^{\min}(\nu_1,\tau(\nu_1)),
    \end{displaymath}
    and (\ref{eq:PNeq}) follows from
    \begin{displaymath}
      \PNtau{1}{+[-]}= \alpha_T^{+[-]}(\nu_1,\tau(\nu_1)).
    \end{displaymath}

    For the induction step, let $c^{\max}:=\max_{s\in\states,a\in\act}C(s,a)<\infty$.
    Let $\delta$ be such that (\ref{eq:ENeq}) and (\ref{eq:PNeq}) hold for
    $N-1$ and $\epsilon':=\min\{\epsilon/(2(1+2c^{\max})),\epsilon/3\}$, and also such that
    for all  $\nu$ with $\delta(\nu)\leq \delta$ and all $a\in\act$:
    $\cost^{\max}(\nu,a)-\cost^{\min}(\nu,a)\leq \epsilon'$,
    and  $\dtv(T(s,a),T(s',a))\leq \epsilon'$ for all $s,s'\in\nu$.
    Let $\partition$ have granularity $\leq\delta$. 
    We then have
    \begin{multline*}
      |\ENtau{N}{+}- \ENtau{N}{-}| = \\
      | \ENtau{N-1}{+}+\sum_{\nu\in\partition} \PNtau{N-1}{+}(\nu) \alpha_C^+(\nu,\tau(\nu)) -
      (\ENtau{N-1}{-}+\sum_{\nu\in\partition} \PNtau{N-1}{-}(\nu) \alpha_C^-(\nu,\tau(\nu)))| \leq \\
      | \ENtau{N-1}{+}-\ENtau{N-1}{-}|  +
      | \sum_{\nu\in\partition} \PNtau{N-1}{+}(\nu) \alpha_C^+(\nu,\tau(\nu)) -
      \sum_{\nu\in\partition} \PNtau{N-1}{-}(\nu) \alpha_C^-(\nu,\tau(\nu)))|
    \end{multline*}
    By induction hypothesis, the left term is bounded by $\epsilon'<\epsilon/2$. According to
    Lemma~\ref{lem:dtvbasic} {\bf A}, the right term is bounded by $\epsilon/2$, thus yielding 
    (\ref{eq:ENeq}).

    The bound  (\ref{eq:PNeq}) directly follows from Lemma~\ref{lem:dtvbasic} {\bf B}.
    \qed
  \end{proof}
  
\end{document}

%% file: macros.tex
\newcommand{\dee}{\mathrm{d}}

\newcommand{\partition}{\mathcal{A}}
\newcommand{\partitionB}{\mathcal{B}}

\newcommand{\states}{\mathcal{S}}
\newcommand{\initial}{s_{init}}
\newcommand{\reals}{\mathbb{R}}

\newcommand{\mdp}{\mathcal{M}}

\newcommand{\borel}{\mathcal{B}}

\newcommand{\run}{\pi}
\newcommand{\runs}{\Pi}
\newcommand{\transitions}{T}

\newcommand{\act}{\mathit{Act}}

\newcommand{\cost}{\mathcal{C}}

\newcommand{\goal}{\mathcal{G}}

\newcommand{\expected}{\ensuremath{\mathbb{E}}}

\newcommand{\adtree}[1][]{\ensuremath{\psi_{#1}}}

\newcommand{\iaction}[1]{\ensuremath{a?}}

\newcommand{\graphover}[1][\adtree]{\ensuremath{\mathtt{Graph}(\adtree)}}

\DeclareMathAlphabet{\mathpzc}{OT1}{pzc}{m}{it}

\newcommand{\pruns}[1]{\ensuremath{P_{#1}}}
\newcommand{\costinfty}{\cost^{\infty}}

\newcommand{\viop}{ \ensuremath{{\cal L}}}
\DeclareMathOperator*{\opt}{opt}
\newcommand{\minmax}{ \ensuremath{\opt}}
\newcommand{\closure}{ \mathit{cl}  }